\documentclass[11pt]{article}

\usepackage{fullpage,times,url,bm}

\usepackage{amsthm,amsfonts,amsmath,amssymb,epsfig,color,float,graphicx,verbatim}
\usepackage{algorithm,algorithmic}
\usepackage{bbm}
\usepackage[square,numbers]{natbib}

\usepackage{hyperref}
\hypersetup{
	colorlinks   = true, 
	urlcolor     = blue, 
	linkcolor    = blue, 
	citecolor   = black 
}

\newtheorem{theorem}{Theorem}[section]
\newtheorem{proposition}{Proposition}[section]
\newtheorem{lemma}{Lemma}[section]

\newtheorem{remark}{Remark}[section]

\newcommand{\reals}{\mathbb{R}}
\newcommand{\E}{\mathbb{E}}

\newcommand{\bx}{\mathbf{x}}
\newcommand{\bw}{\mathbf{w}}

\newcommand{\Fcal}{\mathcal{F}}

\newcommand{\norm}[1]{\|#1\|}
\newcommand{\inner}[1]{\langle#1\rangle}

\newtheorem{claim}{Claim}

\renewcommand{\eqref}[1]{Eq.~(\ref{#1})}

\newcommand\revision[1]{\textcolor{black}{#1}}

\title{On the Power and Limitations of Random Features\\ for Understanding Neural Networks}
\author{Gilad Yehudai\qquad Ohad Shamir\\
Weizmann Institute of Science\\
\texttt{\{gilad.yehudai,ohad.shamir\}@weizmann.ac.il}
}
\date{}
\begin{document}

\maketitle

\begin{abstract}
Recently, a spate of papers have provided positive theoretical results for training over-parameterized neural networks (where the network size is larger than what is needed to achieve low error). The key insight is that with sufficient over-parameterization, gradient-based methods will implicitly leave some components of the network relatively unchanged, so the optimization dynamics will behave as if those components are essentially fixed at their initial random values. In fact, fixing these \emph{explicitly} leads to the well-known approach of learning with random features (e.g. \citep{rahimi2008random,rahimi2009weighted}). In other words, these techniques imply that we can successfully learn with neural networks, whenever we can successfully learn with random features. In this paper, we formalize the link between existing results and random features, and argue that despite the impressive positive results, random feature approaches are also inherently limited in what they can explain. In particular, we prove that under minimal assumptions, random features cannot be used to learn \emph{even a single ReLU neuron} (over standard Gaussian inputs in $\reals^d$ and $\text{poly}(d)$ weights), unless the network size (or magnitude of its weights) is exponentially large in $d$. Since a single neuron with Gaussian inputs \emph{is} known to be learnable with gradient-based methods, we conclude that we are still far from a satisfying general explanation for the empirical success of neural networks. For completeness we also provide a simple self-contained proof, using a random features technique, that one-hidden-layer neural networks can learn low-degree polynomials.
\end{abstract}

\section{Introduction}
Deep learning, in the form of artificial neural networks, has seen a dramatic resurgence in popularity in recent years. This is mainly due to impressive performance gains on various difficult learning problems, in fields such as computer vision, natural language processing and many others. Despite the practical success of neural networks, our theoretical understanding of them is still very incomplete. 

A key aspect of modern networks is that they tend to be very large, usually with many more parameters than the size of the training data: In fact, so many that in principle, they can simply memorize all the training examples (as shown in the influential work of \citet{zhang2016understanding}). The fact that such huge, over-parameterized networks are still able to learn and generalize is one of the big mysteries concerning deep learning. A current leading hypothesis is that over-parameterization makes the optimization landscape more benign, and encourages standard gradient-based training methods to find weight configurations that fit the training data as well as generalize (even though there might be many other configurations which fit the training data without any generalization). However, pinpointing the exact mechanism by which over-parameterization helps is still an open problem.



Recently, a spate of papers (such as \cite{andoni2014learning,daniely2017sgd,du2018gradient,allen2018learning,li2018learning,du2018gradient2,cao2019generalization,allen2018convergence,allen2019can}) provided positive results for training and learning with over-parameterized neural networks. Although they differ in details, they are all based on the following striking observation: When the networks are sufficiently large, standard gradient-based methods change certain components of the network (such as the weights of a certain layer) very slowly, so that if we run these methods for a bounded number of iterations, they might as well be fixed. To give a concrete example, consider one-hidden-layer neural networks, which can be written as a linear combination of $r$ neurons
\begin{equation}\label{eq:network}
N(x) = \sum_{i=1}^r u_i \sigma(\inner{w_i,x}+b_i) ~,
\end{equation}
using weights $\{u_i,w_i,b_i\}_{i=1}^{r}$ and an activation function $\sigma$. When $r$ is sufficiently large, and with standard random initializations, it can be shown that gradient descent will leave the weights $w_i,b_i$ in the first layer nearly unchanged (at least initially). As a result, the dynamics of gradient descent will resemble those where $\{w_i,b_i\}$ are fixed at random initial values -- namely, where we learn a \emph{linear} predictor (parameterized by $u_1,\ldots,u_r$) over a set of $r$ random features of the form $x\mapsto \sigma(\inner{w_i,x}+b_i)$ (for some random choice of $w_i,b_i$). For such linear predictors, it is not difficult to show that they will converge quickly to an optimal predictor (over the span of the random features). This leads to learning guarantees with respect to hypothesis classes which can be captured well by such random features: For example, most papers focus (explicitly or implicitly) on multivariate polynomials with certain constraints on their degree or the magnitude of their coefficients. We discuss these results in more detail (and demonstrate their close connection to random features) in Section \ref{section:analysis of neural networks as random features}. 

Taken together, these results are a significant and insightful advance in our understanding of neural networks: They rigorously establish that sufficient over-parameterization allows us to learn complicated functions, while solving a non-convex optimization problem. However, it is important to realize that this approach can only explain learnability of hypothesis classes which can already be learned using random features. Considering the one-hidden-layer example above, this corresponds to learning linear predictors over a fixed representation (chosen obliviously and  randomly at initialization). Thus, it does not capture any element of \emph{representation learning}, which appears to lend much of the power of modern neural networks. 

In this paper we show that there are inherent limitations on what predictors can be captured with random features, and as a result, on what can be provably learned with neural networks using the techniques described earlier. We consider random features of the form $f_i(x)$ which are chosen from some fixed distribution. The $f_i$'s can be arbitrary functions, including multilayered neural networks and kernel functions, as long as their norm (suitably defined) is not exponential in the input dimension. We show that using $N(x) = \sum_{i=1}^r u_i f_i(x)$ we cannot efficiently approximate \emph{even a single ReLU neuron}: In particular, we can choose a target weight vector $w^*\in \mathbb{R}^d,\norm{\bw^*}=d^2$ and bias term $b^*\in \mathbb{R}$, such that if
    \[
    \E_{\bx}\left[\left(N(x)-[\inner{w^*,x}+b^*]_+\right)^2\right] ~\leq~ \frac{1}{50}
    \]
    (where $[z]_+=\max\{0,z\}$ is the ReLU function and $x$ has a standard Gaussian distribution) then
    \[ r\cdot\max_i |u_i| \geq \exp(\Omega(d)). \]
In other words, either the number of features $r$ or the magnitude of the weights (or both) must be exponential in the dimension $d$. Moreover, if the random features can be written as $x\mapsto f_i(Wx)$ for a random matrix $W$ (which includes for instance vanilla neural networks of any size), then the same result holds for any choice of $w^*$. These results imply that the random features approach cannot fully explain polynomial-time learnability of neural networks, even with respect to data generated by an extremely simple neural network, composed of a single neuron. This is despite the fact that single ReLU neurons with Gaussian inputs \emph{are} easily learnable with gradient-based methods (e.g., \cite{mei2016landscape}, \cite{soltanolkotabi2017learning}\footnote{To be precise, existing theoretical results usually ignore the bias term for simplicity, but we believe this is not a real impediment for the analysis.}). The point we want to make here is that the random feature approach, as a theory for explaining the success of neural networks, cannot explain even the fact that single neurons are learnable.

For completeness we also provide a simple, self-contained analysis, showing how over-parameterized, one-hidden-layer networks can provably learn polynomials with bounded degrees and coefficients, using standard stochastic gradient descent with standard initialization.

We emphasize that there is no contradiction between our positive and negative results: In the positive result on learning polynomials, the required size of the network is exponential in the \emph{degree} of the polynomial, and low-degree polynomials cannot express even a single ReLU neuron if its weights are large enough.

Overall, we argue that although the random feature approach captures important aspects of training neural networks, it is by no means the whole story, and we are still quite far from a satisfying general explanation for the empirical success of neural networks.

\subsection*{Related Work}

The recent literature on the theory of deep learning is too large to be thoroughly described here. Instead, we survey here some of the works most directly relevant to the themes of our paper. In Section \ref{section:analysis of neural networks as random features}, we provide a more technical explanation on the connection of recent results to random features. 

\textbf{The Power of Over-Parameterization.} The fact that over-parameterized networks are easier to train was empirically observed, for example, in \cite{livni2014computational}, and was used in several contexts to show positive results for learning and training neural networks. For example, it is known that adding more neurons makes the optimization landscape more benign (e.g., \cite{safran2016quality,soudry2016no,soudry2017exponentially,safran2017spurious,chizat2018global,soltanolkotabi2019theoretical}), or allows them to learn in various settings (e.g., besides the papers mentioned in the introduction, \cite{brutzkus2017sgd,li2017algorithmic,brutzkus2018over,wang2018learning,du2018power}).

\textbf{Random Features.} The technique of random features was proposed and formally analyzed in \cite{rahimi2008random,rahimi2008uniform,rahimi2009weighted}, originally as a computationally-efficient alternative to kernel methods (although as a heuristic, it can be traced back to the ``random connections'' feature of Rosenblatt's Perceptron machine in the 1950's). These involve learning predictors of the form $x\mapsto \sum_{i=1}^{r}u_i \psi_i(x)$, where $\psi_i$ are random non-linear functions. The training involves only tuning of the $u_i$ weights. Thus, the learning problem is as computationally easy as training linear predictors, but with the advantage that the resulting predictor is non-linear, and in fact, if $r$ is large enough, can capture arbitrarily complex functions. The power of random features to express certain classes of functions has been studied in past years (for example \cite{barron1993universal,rahimi2008uniform,klusowski2018approximation,sun2018random}). However, in our paper we also consider negative rather than positive results for such features. \cite{barron1993universal} also discusses the limitation of approximating functions with a bounded number of such features, but in a different setting than ours (worst-case approximation of a large function class using a fixed set of features, rather than inapproximability of a fixed target function, and not in the context of single neurons). Less directly related, \cite{zhang2016l1} studied learning neural networks using kernel methods, which can be seen as learning a linear predictor over a fixed non-linear mapping. However, the algorithm is not based on training neural networks with standard gradient-based methods. In a very recent work (and following the initial dissemination of our paper), \citet{ghorbani2019linearized} studied the representation capabilities of random features, and showed that in high dimensions random features are not good at fitting high degree polynomials.

\subsection*{Notation}
Denote by $U\left([a,b]^d\right)$ the $d$-dimensional uniform distribution over the rectangle $[a,b]^d$, and by $N(0,\Sigma)$ the multivariate Gaussian distribution with covariance matrix $\Sigma$. For $T\in\mathbb{N}$ let $[T]=\{1,2,\dots,T\}$, and for a vector $w\in \mathbb{R}^d$ we denote by $\|w\|$ the $L_2$ norm. We denote the ReLU function by $[x]_+ = \max\{0,x\}$.

\section{Analysis of Neural Networks as Random Features}\label{section:analysis of neural networks as random features}

In many previous works, a key element is to analyze neural networks as if they are random features, either explicitly or implicitly. Here we survey some of these works and how they can actually be viewed as random features.

\subsection{Optimization with Coupling, Fixing the Output Layer}\label{subsec:optimization with coupling}
One approach is to fix the output layer and do optimization only on the inner layers. Most works that use this method (e.g. \cite{li2018learning}, \cite{du2018gradient2}, \cite{allen2018convergence}, \cite{cao2019generalization}, \cite{allen2018learning}, \cite{allen2019can}) also use the method of "coupling" and the popular ReLU activation. This method uses the following observation: a ReLU neuron can be viewed as a linear predictor multiplied by a threshold function, that is: $[\inner{w,x}]_+ = \inner{w,x}\mathbbm{1}_{\inner{w,x}\geq 0}$. The coupling method informally states that after doing gradient descent with appropriate learning rate and a limited number of iterations, the amount of neurons that change the sign of $\inner{w,x}$ (for $x$ in the data) is small. Thus, it is enough to analyze a linear network over random features of the form: $x\mapsto \inner{w,x}\mathbbm{1}_{\inner{w^{(0)},x}\geq 0}$ where $w^{(0)}$ are randomly chosen.

For example, a one-hidden-layer neural network where the activation $\sigma$ is the ReLU function can be written as
\[ \sum_{i=1}^r u_i^{(t)}\sigma(\inner{w_i^{(t)},x}) = \sum_{i=1}^r u_i^{(t)}\inner{w_i^{(t)},x}\mathbbm{1}_{\inner{w_i^{(t)},x} \geq 0}. \]
Using the coupling method, after doing gradient descent, the amount of neurons that change sign, i.e. the sign of $\inner{w_i^{(t)},x}$ changes, is small. As a result, using the homogeneity of the ReLU function, the following network can actually be analyzed:
\[ 
\sum_{i=1}^r u_i^{(t)}\inner{w_i^{(t)},x}\mathbbm{1}_{\inner{w_i^{(0)},x} \geq 0} = \sum_{i=1}^r \inner{u_i^{(t)}\cdot w_i^{(t)},x}\mathbbm{1}_{\inner{w_i^{(0)},x} \geq 0}, \]
where $w_i^{(0)}$ are randomly chosen. This is just analyzing a linear predictor with random features of the form $x\mapsto x_j\mathbbm{1}_{\inner{w_i^{(0)},x} \geq 0}$. Note that the homogeneity of the ReLU function is used in order to show that fixing the output layer does not change the network's expressiveness. This is not true in terms of optimization, as optimizing both the inner layers and the output layers may help the network converge faster, and to find a predictor which has better generalization properties. Thus, the challenge in this approach is to find functions or distributions that can be approximated with this kind of random features network, using a polynomial number of features. 

\subsection{Optimization on all the Layers}\label{subsec:optimization on all the layers}
A second approach in the literature (e.g. \citet{andoni2014learning}, \citet{daniely2016toward}, \citet{du2018gradient}) is to perform optimization on all the layers of the network, choose a "good" learning rate and bound the number of iterations such that the inner layers stay close to their initialization. For example, in the setting of a one-hidden-layer network, for every $\epsilon >0$, a learning rate $\eta$ and number of iterations $T$ are chosen, such that after running gradient descent with these parameters, there is an iteration $1\leq t\leq T$ such that:
\[ \left\| U^{(t)}\sigma(W^{(t)}x)  -  U^{(t)}\sigma(W^{(0)}x) \right\| \leq \epsilon.\]
Hence, it is enough to analyze a linear predictor over a set of random features:
\[ U^{(t)}\sigma\left(W^{(0)}x\right)=  \sum_{i=1}^r u_i^{(t)}\sigma\left(\inner{w_i^{(0)},x}\right), \]
where $\sigma$ is not necessarily the ReLU function. Again, the difficulty here is finding the functions that can be approximated in this form, where $r$ (the amount of neurons) is only polynomial in the relevant parameters.

\section{Over-Parameterized Neural Networks Learn Polynomials}\label{section:neural networks learn polynomials}

For completeness we provide a simple, self-contained analysis, showing how over-parameterized, one-hidden-layer networks can provably learn polynomials with bounded degrees and coefficients, using standard stochastic gradient descent with standard initialization.  

The data for our network is $(x,y)\in \mathbb{R}^d\times \mathbb{R}$, drawn from an unknown distribution $D$. We assume for simplicity that $\|x\|\leq 1$ and $y=\{-1,+1\}$.

We consider one-hidden-layer feed-forward neural networks which are defined as: 
\[N(x) = N(W,U,x) = U\sigma(Wx), \]
where $\sigma$ is an activation function which acts coordinate-wise and $W\in \mathbb{R}^{r\times d},\ U\in \mathbb{R}^r$. We will also use the following form for the network:
\begin{equation}
N(x) = \sum_{i=1}^r u_i \sigma(\inner{w_i,x})
\end{equation}
here $u_i\in \mathbb{R}$ and $w_i\in \mathbb{R}^d$. 

For simplicity we will use the hinge loss, which is defined by: $l(\hat{y},y) = \max\{0,1-\hat{y}y\}$, thus the optimization will be done on the function $l(N(x),y) = l(N(W,U,x),y)$. We will also use the notation:
\[ L_D(W,U) = \mathbb{E}_{(x,y)\sim D} \left[l(N(W,U,x),y)\right] \] 

We will use the standard form of SGD to optimize $L_D$, where at each iteration a random sample $(x,y)$ is drawn from $D$ and we update:
\[W_{i+1} = W_i - \eta \frac{\partial l(N(W_i,U_i,x_i),y_i)}{\partial W_i}\]
\[U_{i+1} = U_i - \eta \frac{\partial l(N(W_i,U_i,x_i),y_i)}{\partial U_i} \]

The initialization of $W_0$ is a standard Xavier initialization \cite{glorot2010understanding}, that is $w_i \sim U\left(\left[\frac{-1}{\sqrt{d}},\frac{1}{\sqrt{d}}\right]^d\right)$. $U_0$ can be initialized in any manner, as long as its norm is smaller than $\frac{1}{\sqrt{r}}$, e.g. we can initialize $U_0$ = 0. This kind of initialization for the outer layer has been used also in other works (see \cite{daniely2017sgd}, \cite{andoni2014learning}).

The main result of this section is the following:
\begin{theorem}\label{Main theorem of the paper with analytic activation}
Let $d\in\mathbb{N}$ be the inputs dimension. For some $\tilde{n}\in\mathbb{N}$, let $\sigma:\mathbb{R}\rightarrow \mathbb{R}$ be an analytic activation function with Taylor coefficients $\left\{a_i\right\}_{i=0}^\infty$, which is $L$-Lipschitz with $\sigma(0)\leq L$ and satisfies that $a_n \leq \left(40d^{1.5}\right)^{-n}$ for every $n\geq \tilde{n}$. Let $D$ be any distribution over the labelled data $(x,y)\in \mathbb{R}^d\times \mathbb{R}$ with $\|x\|\leq 1,\ y\in\{-1,+1\}$, and let $\epsilon >0,\ \delta > 0$, $\alpha >0$, and $k$ be some positive integer. Suppose we run SGD on the neural network:
\[ N(W,U,x) = U\sigma (Wx)= \sum_{i=1}^r u_i\sigma\left(\inner{w_i,x}\right) \]
with the following parameters:
\begin{enumerate}
\item $r$ neurons with $r \geq \frac{64\beta^6L^2}{\epsilon^4}log\left(\frac{1}{\delta}\right)$
\item $W_0$ is initialized with $w_i\sim U\left( \left[\frac{-1}{\sqrt{d}},\frac{1}{\sqrt{d}}\right]^d\right)$ for $i=1,\dots,r$ and $U_0$ is initialized s.t $\|U_0\|\leq \frac{1}{\sqrt{r}}$ 
\item learning rate $\eta = \frac{\epsilon}{8r}$
\item $T$ steps with $T= \frac{4\beta^2}{\epsilon^2}$~,
\end{enumerate}
where $\beta = \left(5120d^{3.5}\right)^{\max\{k,\tilde{n}\}}\left(\frac{\alpha}{a}\right)^2\epsilon^{-1}$
and $a = \min\{a_1,\dots,a_k\}$.
Then, for every polynomial $P(x_1,\dots,x_d)$ with  $\deg(P) \leq k$, the coefficients of $P$ are bounded by $\alpha$ and all the monomials of $P$ which have a non-zero coefficient also have a non-zero coefficient in the Taylor series of $\sigma$, w.p $ > 1-\delta$ over the initialization there is $t\in [T]$ such that:
\[ \mathbb{E}\left[ L_D\left(W_t,U_t\right) \right] \leq L_D(P(x)) + \epsilon. \]
Here the expectation is over the random choice of $(x_i,y_i)$ in each iteration of SGD.
\end{theorem}

We note that for simplicity, we focused on analytic activation functions, although it is possible to derive related results for non-analytic activations such as a ReLU (see Appendix \ref{app:relu} for a discussion). The assumptions on the Taylor coefficients of the activation function includes for example the $\exp$ and $\rm{erf}$ activations. We note that similar assumptions were used also in other works on learning polynomials (e.g. \cite{ghorbani2021linearized}).
Also, note that we did not use a bias term in the architecture of the network in the theorem (namely, we have $\sigma(\inner{w_i,x})$ and not $\sigma(\inner{w_i,x} + b_i)$). This is because if the polynomial we are trying to compete with has a constant factor, then we require that the Taylor expansion of the activation also has a constant factor, thus the bias term is already included in the Taylor expansion of the activation function.

\begin{remark}
Suppose we are given a sample set $S=\{(x_i,y_i)\}_{i=1}^m$. By choosing $D$ uniform on the sample set $S$, Theorem \ref{Main theorem of the paper with analytic activation} shows that SGD over the sample set will lead to an average loss not much worse than the best possible polynomial predictor with bounded degree and coefficients.
\end{remark}

At high level, the proof idea of Theorem \ref{Main theorem of the paper with analytic activation} is divided into three steps. In the first step we show that with an appropriate learning rate and limited amount of iterations, neural networks generalize better than random features. This step allows us to focus our analysis on the behaviour of a linear combination of random features instead of the more complicated architecture of neural networks. In the second step using McDiarmid's theorem we show that by taking enough random features, they concentrate around their expectation. In the third step we use Legendre's polynomials to show that any polynomial can be \revision{approximated by an expectation}
of random features.

First we introduce some notations regarding multi-variable polynomials: Letting $J=(j_1,\dots,j_d)$ be a multi index, and given $x\in \mathbb{R}^d$, we define $x^J = x_1^{j_1}\cdots x_d^{j_d}$, and also $|J|=j_1+\dots+j_d$. We say for two multi indices $J',J$ that $J'\leq J$ if for all $1\leq i \leq d$, $j_i'\leq j_i$ and that $J' < J$ if $J'\leq J$ and also there is an index $1\leq s \leq d$ such that $j_s' < j_s$. For $k\in \mathbb{N}$ and multi index $J=(j_1,\dots,j_d)$ we say that $J\leq k$ if $j_1+\dots +j_d \leq k$. Lastly, given a multi-variable polynomials $P(x) = \sum_J c_Jx^J$, where $c_J\in \mathbb{R}$ we define:
\[ |P| = \max_J |c_J|.\]

We break the proof to three steps, where each step contains a theorem which is independent of the other steps. Finally we combine the three steps to prove the main theorem.

\subsubsection*{Step 1: SGD on Over-Parameterized Networks Competes with Random Features}

Recall we use a network of the form:
\begin{equation}\label{outline the form of the network}
N(W,U,x) = U\sigma(Wx) = \sum_{i=1}^r u_i \sigma(\inner{w_i,x}), 
\end{equation}
where $W_0,\ U_0$ are initialized as described in the theorem
We show that for any target matrix $U^*$ with a small enough norm and every $\epsilon > 0$, if we run SGD on $l(N(W_0,U_0,x),y)$ with appropriate learning rate $\eta$ and number of iterations $T$, there is some $t\in [T]$ with:
\begin{equation}\label{outline generalization bound U t W t U * W 0}
\mathbb{E}\left(L_D(W_t,U_t)\right)\leq L_D(W_0,U^*) + \epsilon 
\end{equation}
where the expectation is over the random choices of examples in each round of SGD.\\
The bound in \eqref{outline generalization bound U t W t U * W 0} means that SGD on randomly initialized weights competes with random features. By random features here we mean any linear combination of neurons of the form $\sigma(\inner{w_i,x})$ where the $w_i$ are randomly chosen, and the norm of the weights of the linear combination are bounded.
In more details:
\begin{theorem}\label{theorem about optimization bound related to W 0 U 0 }
  Assume we initialize $U_0, W_0$ such that $\|U_0\| \leq \frac{1}{\sqrt{r}}$ and $\|W_0\| \leq \sqrt{r}$. Also assume that $\sigma$ is $L$-Lipschitz with $\sigma(0) \leq L$, and let $C \geq 1$ be a constant.  Letting $\epsilon >0$, we run SGD with step size $\eta = \frac{\epsilon}{8r}$ and $T$ steps with $T = \frac{4C^2}{\epsilon^2}$ and let $W_1,\dots,W_T$ be the weights produced at each step. If we pick $r$ such that $r \geq  \frac{64C^6 L^2}{\epsilon^4} $ then for every target matrix  $U^*$ with $\|U^*\|\leq \frac{C}{\sqrt{r}}$ there is a $t\in [T]$ s.t:
  \[ \mathbb{E}\left[L_D(W_t, U_t)\right] \leq L_D(W_0,U^*) + \epsilon. \]
Here the expectation is over the random choice of the training examples in each round of SGD.
\end{theorem}

In the proof of Theorem \ref{theorem about optimization bound related to W 0 U 0 } we first show that for the chosen learning rate $\eta$ and limited number of iterations $T$, the matrix $W$ does not change much from its initialization. After that we use results from online convex optimization for linear prediction with respect to $U^*$ with a sufficiently small norm to prove the required bound. For a full proof see Appendix \ref{appendix optimization with respect to a target matrix}. Note that in the theorem we did not need to specify the initialization scheme, only to bound the norm of the initialized weights. The optimization analysis is similar to the one done in Daniely \cite{daniely2017sgd}.

\subsubsection*{Step 2: Random Features Concentrate Around their Expectation}

In the previous step we showed that in order to bound the expected loss of the network, it is enough to consider a network of the form $\sum_i u_i \sigma(\inner{w_i,x})$, where the $w_i$ are randomly initialized with $w_i\sim U\big( [\frac{-1}{\sqrt{d}},\frac{1}{\sqrt{d}}]^d\big)$. We now show that if the number of random features $r$ is large enough, then a linear combination of them approximates functions of the form $x\mapsto \E_{w}[\sigma(\inner{w,x})g(w)]= c_d \int_{w\in \left[\frac{-1}{\sqrt{d}},\frac{1}{\sqrt{d}}\right]^d}g(w)\sigma(\inner{w,x})dw$ for an appropriate normalization factor $c_d$:



\begin{theorem}\label{theorem about approximating integral as a sum}
  Let $f(x)= c_d \int_{w\in \left[\frac{-1}{\sqrt{d}},\frac{1}{\sqrt{d}}\right]^d}g(w)\sigma(\inner{w,x})dw$ where $\sigma:\mathbb{R}\rightarrow \mathbb{R}$ is $L$-Lipschitz on $[-1,1]$ with $\sigma(0) \leq L$, and $c_d = \left( \frac{\sqrt{d}}{2} \right)^d$ a normalization term. Assume that $max_{\|w\|\leq 1} |g(w)| \leq C$ for a constant $C$. Then for every $\delta >0$ if $w_1,\dots,w_r$ are drawn i.i.d from the uniform distribution on  $\left[\frac{-1}{\sqrt{d}},\frac{1}{\sqrt{d}}\right]^d$ , w.p $> 1-\delta$ there is a function of the form
  \[ \hat{f}(x) = \sum_{i=1}^{r}u_i \sigma(\inner{w_i,x}) \]
  where $|u_i| \leq \frac{C}{r}$ for every $1\leq i \leq r$, such that:
  \[ \sup_x \left|\hat{f}(x) - f(x) \right| \leq \frac{LC}{\sqrt{r}}\left(4 + \sqrt{2 \log\left(\frac{1}{\delta}\right)}\right) \]
\end{theorem}

Theorem \ref{theorem about approximating integral as a sum} basically states that random features concentrate around their expectation, and the rate of convergence is $O\left(\frac{1}{\sqrt{r}}\right)$ where $r$ is the amount of random features that were sampled. The proof is based on concentration of measure and Rademacher complexity arguments, and appears in Appendix \ref{appendix approximating integral to a discrete sum}.

\subsubsection*{Step 3: \revision{Approximating} Polynomials \revision{Using} Expectation of Random Features}

In the previous step we showed that random features can approximate functions with the integral form:
\[ f(x) = c_d \int_{w\in \left[\frac{-1}{\sqrt{d}},\frac{1}{\sqrt{d}}\right]^d}g(w)\sigma(\inner{w,x})dw \]
In this step we show how a a polynomial $P(x)$ with bounded degree and coefficients can be \revision{approximated by}
this form. This means that we need to find a function $g(w)$ for which $f(x) = P(x)$. To do so we use the fact that $\sigma(x)$ is analytic, thus it can be represented as an infinite sum of monomials using a Taylor expansion, and take $g(w)$ to be a finite weighted sum of Legendre polynomials, which are orthogonal with respect to the appropriate inner product. The main difficulty here is to find a bound on $\max_{w\in \left[\frac{-1}{\sqrt{d}},\frac{1}{\sqrt{d}}\right]^d} |g(w)|$, which in turn also bounds the distance between the sum of the random features and its expectation. The main theorem of this step is:

\begin{theorem}\label{theorem about integral representation of polynomials using legendre}
For some $\tilde{n}\in\mathbb{N}$, let $\sigma:\reals\rightarrow\reals$ be an analytic function with Taylor expansion coefficients $\left\{a_i\right\}_{i=0}^\infty$ which satisfies $a_n \leq \left(40d^{1.5}\right)^{-n}$ for every $n\geq \tilde{n}$.  Let $P(x)=\sum_{J}\alpha_J x^J$ be a polynomial, where $x\in\reals^d$, and all the monomials of $P$ which have non-zero coefficient also have a non-zero coefficient in the Taylor series of $\sigma$. Assume $\text{deg}(P)\leq k$, $|\alpha_J| \leq \alpha$ for some $\alpha \geq 1$ ,and denote $a=\min\{a_1,\dots,a_k\}$. Then, for any $\epsilon > 0$ there exists a function $g(w):\reals^d\rightarrow\reals$ that satisfies the following:

\begin{enumerate}
  \item \label{integral representation of sum of powers}
  $\max_{\norm{x}\leq 1}\left|c_d\int_{w\in \left[-\frac{1}{\sqrt{d}},\frac{1}{\sqrt{d}}\right]}\sigma(\inner{w,x}) g(w)dw - P(x)\right| < \epsilon$

  \item \label{bound on g(w)} $\max_{w\in \left[-\frac{1}{\sqrt{d}},\frac{1}{\sqrt{d}}\right]} |g(w)| \leq \left(5120d^{3.5}\right)^{\max\{k,\tilde{n}\}}\left(\frac{\alpha}{a}\right)^2\epsilon^{-1}$ 
\end{enumerate}
where $c_d = \left(\frac{\sqrt{d}}{2}\right)^d$ is a normalization term.
\end{theorem}


For a full proof of Theorem \ref{theorem about integral representation of polynomials using legendre} and an overview of Legendre polynomials see Appendix \ref{appendix about representing polynomials in an integral form}.

\subsubsection*{Step 4: Putting it all Together}
We are now ready to prove the main theorem of this section. The proof is done for convenience in reverse order of the three steps presented above.

\begin{proof}[Proof of Theorem \ref{Main theorem of the paper with analytic activation}]
Let $a_0,a_1,\dots,a_k$ be the coefficients of the Taylor expansion of $\sigma$ up to degree $k$, and let $P(x)$ be a a polynomial with $\deg(P)\leq k$ and $|P|\leq \alpha$, such that if $a_j=0$ then the monomials in $P(x)$ of degree $j$ also have a zero coefficient.

First, we use Theorem \ref{theorem about integral representation of polynomials using legendre} to find a function $g(w)$ such that:

\begin{equation}\label{main theorem polynomial as an integral}
    \max_{\norm{x}\leq 1}\left|c_d\int_{w\in \left[-\frac{1}{\sqrt{d}},\frac{1}{\sqrt{d}}\right]}\sigma(\inner{w,x}) g(w)dw - P(x)\right| < \epsilon
\end{equation}
Then we consider drawing random features $w_1,\dots, w_r \sim U\left(\left[\frac{-1}{\sqrt{d}},\frac{1}{\sqrt{d}}\right]^d\right)$ i.i.d. Using Theorem \ref{theorem about approximating integral as a sum}, the choice of $r$ and \eqref{main theorem polynomial as an integral}, w.p $>1-\delta$ there is $U^* = (u_1,\dots,u_r)$ such that:
\begin{align}\label{discrete sum and polynomial are close}
 &=  \max_{\norm{x}\leq 1} \left| \sum_{i=1}^r u_i\sigma(\inner{w_i,x}) - c_d\int_{w\in \left[\frac{-1}{\sqrt{d}},\frac{1}{\sqrt{d}}\right]^d}\sigma(\inner{w,x}) g(w)\right| \leq \epsilon,
\end{align}
and also $|u_i| \leq \max_{\|w\|\leq 1} \frac{|g(w)|}{r} \leq \alpha^k \left(\frac{A}{a}\right)^k (12d)^{2k^2}$, thus $\|U^*\| \leq \frac{\alpha^k \left(\frac{A}{a}\right)^k (12d)^{2k^2}} {\sqrt{r}}$. 

Finally, we use Theorem \ref{theorem about optimization bound related to W 0 U 0 } with the defined learning rate $\eta$ and iterations $T$ to find $t\in [T]$ such that:
\begin{equation}\label{almost final approximation, only neede the polynomial}
\mathbb{E}\big( L_D(W_t,U_t) \big) \leq L_D(W_0,U^*) + \epsilon.
\end{equation}
Combining \eqref{main theorem polynomial as an integral}, \eqref{discrete sum and polynomial are close} with \eqref{almost final approximation, only neede the polynomial} gives:
\[ \mathbb{E}\big( L_D(W_t,U_t) \big) \leq L_D(P(x)) + 3\epsilon. \]
Re-scaling $\epsilon$ finishes the proof.
\end{proof}

\section{Limitations of Random Features}\label{section limitations of random features}

Having discussed and shown positive results for learning using (essentially) random features, we turn to discuss the limitations of this approach.

Concretely, we will consider in this section data $(x,y)$, where $x\in \reals^d$ is drawn from a standard Gaussian on $\reals^d$, and there exists some single ground-truth neuron which generates the target values $y$: Namely, $y=\sigma(\inner{w^*,x}+b^*)$ for some fixed $w^*,b^*$. We also consider the squared loss $l(\hat{y},y)=(\hat{y}-y)^2$, so the expected loss we wish to minimize takes the form
\begin{equation}\label{eq:targetsquared}
\E_{x}\left[\left(\sum_{i=1}^{r}u_i f_i(x)-\sigma(\inner{w^*,x}+b^*)\right)^2\right]
\end{equation}
where $f_i$ are the random features. Importantly, when $r=1$, $\sigma$ is the ReLU function, and $f_1(x) = \sigma(\inner{w,x}+b)$ (that is, we train a single neuron with Gaussian inputs to learn a single target neuron), this problem \emph{is} quite tractable with standard gradient-based methods (see, e.g., \cite{mei2016landscape}, \cite{soltanolkotabi2017learning}). In this section, we ask whether this positive result -- that single target neurons can be learned -- can be explained by the random features approach. Specifically, we consider the case where the function $f_i$ are arbitrary functions chosen obliviously of the target neuron (e.g. multilayered neural networks at a standard random initialization), and ask what conditions on $r$ and $u_i$ are required to minimize \eqref{eq:targetsquared}. Our main results (Theorem \ref{thm:limitations of random features} and Theorem \ref{thm:limitations of random features with coupling}) show that either one of them has to be exponential in the dimension $d$, as long as the sizes of $w^*,b^*$ are allowed to be polynomial in $d$. Since networks with exponentially-many neurons or exponentially-sized weights are generally not efficiently trainable, we conclude that an approach based on random-features cannot explain why learning single neurons is tractable in practice. In Theorem \ref{thm:limitations of random features}, we show the result for \emph{any} choice of $w^*$ with some fixed norm, but require the feature functions to have a certain structure (which is satisfied by neural networks). In Theorem \ref{thm:limitations of random features with coupling}, we drop this requirement, but then the result only holds for a particular $w^*$.

To simplify the notation in this section, we consider functions on $x$ as elements of the $L^2$ space weighted by a standard Gaussian measure, that is
\begin{align*}
 &\|f(x)\|^2 := \mathbb{E}_x\left[f^2(x)\right]= c_d\int_{\mathbb{R}^d} f^2(x)e^{\frac{-\|x\|^2}{2}}dx, \\ 
 & \inner{f(x),g(x)} := \mathbb{E}_x[f(x)g(x)] = c_d\int_{\mathbb{R}^d} f(x)g(x)e^{\frac{-\|x\|^2}{2}}dx, 
\end{align*}
where $c_d=\left(\frac{1}{\sqrt{2\pi}}\right)^d$ is a normalization term. For example, \eqref{eq:targetsquared} can also be written as $\norm{\sum_{i=1}^{r}u_i f_i(x)-\sigma(\inner{w^*,x}+b^*)}^2$. 

\subsection{Warm up: Linear predictors}

Before stating our main results, let us consider a particularly simple case, where $\sigma$ is the identity, and our goal is to learn a \emph{linear predictor} $\bx\mapsto \inner{w^*,x}$ with $\norm{w^*}=1$. We will show that already in this case, there is a significant cost to pay for using random features. The main result in the next subsection can be seen as an elaboration of this idea.

In this setting, finding a good linear predictor, namely minimizing
$\left\|\inner{w,x} - \inner{w^*,x}\right\|$
is easy: It is a convex optimization problem, and is easily solved using standard gradient-based methods.
Suppose now that we are given random features $w_1,\dots,w_r\sim N\left(0,\frac{1}{d}I_d\right)$ and want to find $u_1,\dots,u_r\in \mathbb{R}$ such that:
\begin{equation}\label{linear predictor random features}
    \left\|\sum_{i=1}^r u_i\inner{w_i,x} - \inner{w^*,x}\right\| \leq \epsilon. 
\end{equation}
The following proposition shows that with high probability, \eqref{linear predictor random features} cannot hold unless $r = \Omega(d)$. This shows that even for linear predictors, there is a price to pay for using a combination of random features, instead of learning the linear predictor directly.

\begin{proposition}\label{prop:warm up linear predictor}
Let $w^*$ be some unit vector in $\reals^d$, and suppose that we pick random features $w_1,\ldots,w_r$ i.i.d. from any spherically symmetric distribution in $\reals^d$. If $r\leq \frac{d}{2}$, then with probability at least $1-\exp(-cd)$ (for some universal constant $c>0$), for any choice of weights $u_1,\ldots,u_r$, it holds that
\[
\left\| \sum_{i=1}^r u_i\inner{w_i,x} - \inner{w^*,x} \right\|^2 ~\geq~ \frac{1}{4}
.\]
\end{proposition}

The full proof appears in Appendix \ref{appendix proofs from section limitations of random features}, but the intuition is quite simple: With random features, we are forced to learn a linear predictor in the span of $w_1,\ldots,w_r$, which is a random $r$-dimensional subspace of $\reals^d$. Since this subspace is chosen obliviously of $w^*$, and $r\leq \frac{d}{2}$, it is very likely that $w^*$ is not close to this subspace (namely, the component of $w^*$ orthogonal to this subspace is large), and therefore we cannot approximate this target linear predictor very well. 

\subsection{Features Based on Random Linear Transformations}
Having discussed the linear case, let us return to the case of a non-linear neuron. Specifically, we will show that even a single ReLU neuron cannot be approximated by a very large class of random feature predictors, unless the amount of neurons in the network is exponential in the dimension, or the coefficients of the linear combination are exponential in the dimension. In more details:

\begin{theorem}\label{thm:limitations of random features}
There exist universal constants $c_1,c_2,c_3$ such that the following holds. Let $d>c_1$, $k\in\mathbb{N}$, and let $\mathcal{F}$ be a family of functions from $\reals^k$ to $\reals$. Also, let $W\in\mathbb{R}^{k\times d}$ be a random matrix whose rows are sampled uniformly at random from the unit sphere. Suppose that $f_W(x):=f(Wx)$ satisfies $\|f_W\|\leq \exp(c_2 d)$ for any realization of $W$ and for all $f\in\mathcal{F}$. Then there exists $b^*\in\mathbb{R}$ with $|b^*|\leq 6d^4 + 1$ such that for every $w^*\in\mathbb{R}^d$ with $\|w^*\|=d^3$, and for any $f_1,\dots f_r\in \mathcal{F}$, w.p $> 1- \exp(-c_3 d)$ over sampling $W$,  if
\[ \left\|\sum_{i=1}^r u_i f_i(Wx) - \left[\inner{w^*,x} + b^*\right]_+\right\|^2 \leq \frac{1}{50}~, \]
then
\[ r \cdot \max_i |u_i| \geq \frac{1}{200d^4}\exp\left(c_3 d\right). \]
\end{theorem}

Note that the theorem allows any ``random'' feature which can be written as a composition of some function $f$ (chosen randomly or not), and a random linear transformation. This includes as special cases one-hidden layer networks ($f_i(Wx)=\sigma(\inner{W_i,x})$, with each $W_i$ chosen randomly), or multi-layer neurons of any depth (as long as the first layer performs a random linear transformation). 

\begin{remark}
As opposed to the linear case, here we also have a restriction on $\max_i|u_i|$. We conjecture that it is possible to remove this dependence and leave it to future work. 
\end{remark}

To prove the theorem, we will use the following proposition, which implies that functions of the form $x\mapsto \psi(\inner{w,x})$ for a certain sine-like $\psi$ and "random" $w$ are nearly uncorrelated with any fixed function.

\begin{proposition}\label{proposition properties of psi}
Let $d\in \mathbb{N}$, where $d > c'$ for some positive constant $c'$, and let $a = 6d^2 + 1$. Define the following function:
\[ \psi(x) = [x + a]_+ +\sum_{n=1}^a 2[x+a-2n]_+(-1)^n -1, \]
Then $\psi(x)$ satisfies the following:
\begin{enumerate}
\item\label{first item of psi theorem} It is a periodic odd function on the interval $[-a,a]$
\item\label{second item of psi theorem} For every $w^*$ with $\|w^*\| = d$, 
$ \left\|\psi(\inner{w^*,x})\right\|^2 \geq \frac{1}{6} $.
\item\label{third item of psi theorem} For every $f\in L^2(\mathbb{R}^d)$, we have
$\mathbb{E}_w\left(\inner{f(x),\psi\left(\inner{w,x}\right)}^2\right) \leq 20 \|f\|^2\cdot \exp(-cd)$,
where $w$ is sampled uniformly from $\{w:\ \|w\|=d\}$, and $c>0$ is a universal constant.
\end{enumerate}
\end{proposition}

Items \ref{first item of psi theorem} and \ref{second item of psi theorem} follow by a straightforward calculation, where in item \ref{second item of psi theorem} we also used the fact that $x$ has a symmetric distribution. Item \ref{third item of psi theorem} relies on a claim from \cite{shamir2018distribution}, which shows that periodic functions of the form $x\mapsto \psi(\inner{w,x})$ for a random $w$ with sufficiently large norm  have low correlation with any fixed function. The full proof can be found in Appendix \ref{appendix proofs from section limitations of random features}. 

At a high level, the proof of Theorem \ref{thm:limitations of random features} proceeds as follows: If we choose and fix $\{f_i\}_{i=1}^{r}$ and $W$, then any linear combination of random features $f_i(Wx)$ with small weights will be nearly uncorrelated with $\psi(\inner{w^*,x})$, in expectation over $w^*$ . But, we know that $\psi(\inner{w^*,x})$ can be written as a linear combination of ReLU neurons, so there must be some ReLU neuron which will be nearly uncorrelated with any linear combination of the random features (and as a result, cannot be well-approximated by them). Finally, by a symmetry argument, we can actually fix $w^*$ arbitrarily and the result still holds. We now turn to provide the formal proof:

\begin{proof}[Proof of Theorem \ref{thm:limitations of random features}]
Take $\psi(x)$ from Proposition \ref{proposition properties of psi} and denote for $w\in \mathbb{R}^d$, $\psi_w(x) = \psi(\inner{w,x})$.  If we sample $w^*$ uniformly from $\{w:\ \|w\|=d\}$, then for all $f\in\mathcal{F}$:
\begin{equation*}\label{eq:expectation over w^* of inner product}
   \mathbb{E}_{w^*} \left[\left|\inner{ f_W,\psi_{w^*}}\right|\right] \leq 20\|f_W\|^2\exp(-cd) \leq \exp(-c_3 d), 
\end{equation*}
where $c_3$ is a universal constant that depends only on the constant $c$ from Proposition \ref{proposition properties of psi} and on $c_2$. Hence also:
\[
\mathbb{E}_{w^*} [\mathbb{E}_W\left[\left|\inner{ f_W,\psi_{w^*}}\right|\right]] \leq \exp(-c_3 d) 
\]
Therefore, there is $w^*$ with $\|w^*\| = d$ such that:
\begin{equation}\label{bound w.h.p on <f i , psi>}
    \mathbb{E}_W[|\inner{f_W,\psi_{w^*}}|]\leq \exp(-c_3d) 
\end{equation} 
Using Markov's inequality and dividing $c_3$ by a factor of $2$, we get w.p $>1-\exp\left(-c_3d\right)$ over sampling of $W$,
$
|\inner{f_W,\psi_{w^*}}|\leq \exp(-c_3d) 
$ for the $w^*$ we found above.
Finally, if we pick $f_1,\dots,f_r\in \mathcal{F}$, using the union bound we get that w.p $>1-r\exp\left(-c_3d\right)$ over sampling of $W$:
\[
\forall i\in\{1,\dots,r\},\ |\inner{f_{i_W},\psi_{w^*}}|\leq \exp(-c_3d) 
\]

We can write $\psi(x) = \sum_{j=1}^{a}a_j[x + c_j]_+ - 1$, where $a=6d^2 + 1$ and $|a_j|\leq 2$, $c_j\in [-a,a]$ for $j=1,\dots a$. Let $w^*\in \mathbb{R}^d$ with $\|w^*\| = d$ be as above, and denote $f^*_j(x) = [\inner{w^*,x} + c_j]_+$. Assume that for every $j$ we can find $u^j \in \mathbb{R}^r$ and $f_1,\dots,f_r\in \mathcal{F}$ such that
$\left\|\sum_{i=1}^r u_i^jf_{i_W} - f_j^*\right\|^2 \leq \epsilon$,
where $W$ is distributed as above. Let $f_0(Wx) = b_0$ the bias term of the output layer of the network, then also:
\begin{align}\label{upper bound on the norm between psi and random features combination}
&\left\|\sum_{i=0}^r\left(\sum_{j=1}^{a}u_i^ja_j\right)f_i(Wx) - \sum_{j=1}^{a} a_jf^*_j(x) - 1\right\|^2 = \left\|\sum_{i=0}^r\widetilde{u}_if_i(Wx) - \psi(\inner{w^*,x})\right\|^2 \nonumber \\
\leq &  \left(\sum_{j=1}^{a} \epsilon|a_j|^2\right)^2 \leq (4\epsilon a)^2 = 16\epsilon^2 a^2
\end{align} 
where $\widetilde{u}_i = \left(\sum_{j=1}^{a}u_i^ja_j\right)$. On the other hand using \eqref{bound w.h.p on <f i , psi>} and item \ref{second item of psi theorem} from Proposition \ref{proposition properties of psi} we get w.p $>1-r\exp\left(-c_3d\right)$ over the distribution of $W$ that:
\begin{align}\label{lower bound on on the norm between psi and random features combination}
& \left\|\sum_{i=0}^r\widetilde{u}_if_i(Wx) - \psi(\inner{w^*,x})\right\|^2 \geq \|\psi_{w^*}\|^2 - 2\left|\inner{\sum_{i=0}^r \widetilde{u_i}f_{i_W},\psi_{w^*}}\right|\nonumber  \\
 \geq &\frac{1}{6} - 2\max_i|\widetilde{u_i}| \sum_{i=1}^r |\inner{f_{i_W},\psi_{w^*}}| \geq \frac{1}{6} - 2\max_i|\widetilde{u_i}| r \exp\left(-c_3 d\right) \nonumber\\
  \geq & \frac{1}{6} - 4a\max_i|u^j_i| r \exp\left(-c_3 d\right),
\end{align}
where the last inequality is true for all $j=1,\dots,a$. Combining \eqref{upper bound on the norm between psi and random features combination} and \eqref{lower bound on on the norm between psi and random features combination}, there are $w^*\in \mathbb{R}^d$ with $\|w^*\|=d$ and $b^*\in [-a,a]$ such that w.p $>1-r\exp\left(-c_3d\right)$ over the sampling of $W$, if there is $u\in \mathbb{R}^r$ that:
\begin{equation}\label{upper bound by epsilon on a single relu1}
    \left\|\sum_{i=1}^r u_i f_i(Wx) - [\inner{w^*,x} + b^*]_+ \right\|^2 \leq \epsilon 
\end{equation}  
then
$r\max_i|u_i| \geq \big(1-576 d^4 \epsilon^2)\frac{1}{144d^2}\exp\left(c_3 d\right)$.

We now show that the above does not depend on $w^*$. Any $w\in \mathbb{R}^d$ with $\|w\|=d$ can be written as $Mw^*$ for some orthogonal matrix $M$. Now:
\begin{align}\label{upper bound by epsilon symmetry argument}
    &\left\|\sum_{i=1}^r u_i f_i(Wx) - [\inner{w^*,x} + b^*]_+ \right\| = \left\|\sum_{i=1}^r u_i f_i(WM^\top Mx) - [\inner{Mw^*,Mx} + b^*]_+ \right\|\\
    & =  \mathbb{E}_x\left[\sum_{i=1}^r u_i f_i(WM^\top Mx) - [\inner{Mw^*,Mx} + b^*]_+ \right] = \mathbb{E}_x\left[\sum_{i=1}^r u_i f_i(WM^\top x) - [\inner{Mw^*,x} + b^*]_+ \right] \\
    & =  \left\|\sum_{i=1}^r u_i f_i(WM^\top x) - [\inner{Mw^*,x} + b^*]_+ \right\|,
\end{align}
where we used the fact that $x$ have a spherically symmetric distribution. We also use the fact that the rows of $W$ have a spherically symmetric distribution to get the same probability as above, thus we can replace $WM^\top$ by $W$ to get the following:
There exists $b^*\in [-a,a]$ such that for all $w^*\in \mathbb{R}^d$ with $\|w^*\|=d$, w.p $>1-r\exp\left(-c_3d\right)$ over the sampling of $W$, if there is $u\in \mathbb{R}^r$ that:
\begin{equation}\label{upper bound by epsilon on a single relu}
    \left\|\sum_{i=1}^r u_i f_i(Wx) - [\inner{w^*,x} + b^*]_+ \right\|^2 \leq \epsilon 
\end{equation}  
then
$r\max_i|u_i| \geq \big(1-576 d^4 \epsilon^2)\frac{1}{144d^2}\exp\left(c_3 d\right)$.

Lastly, by multiplying both sides of \eqref{upper bound by epsilon on a single relu} by $d^2$, using the homogeneity of ReLU and setting $\epsilon = \frac{1}{50d^2}$ we get that there is $\hat{b}^*\in [-6d^4-1, 6d^4+1]$ such that for every $\hat{w}^*\in\mathbb{R}^d$ with $\|\hat{w}^*\| = d^3$, w.p $>1-r\exp\left(-c_3d\right)$ over the sampling of $W$, if there is $\hat{u}\in \mathbb{R}^r$ such that
\[ \left\|\sum_{i=1}^r \hat{u}_i f_i(Wx) - [\inner{\hat{w}^*,x} + \hat{b}^*]_+ \right\|^2 \leq \frac{1}{50}, \] 
then
$r\max_i|\hat{u}_i| \geq \frac{1}{200d^4}\exp\left(c_3 d\right)$.
\end{proof}

\begin{remark}
It is possible to trade-off between the norm of $w^*$ and the required error. This can be done by altering the proof of Theorem \ref{thm:limitations of random features}, where instead of multiplying both sides of \eqref{upper bound by epsilon on a single relu} by $d$ we could have multiplied both sides by $\alpha$. This way the following is proved: With the same assumptions as in Theorem \ref{thm:limitations of random features}, for all $w^*\in \mathbb{R}^d$ with $\|w^*\|=\alpha d$ there is $b^*\in \mathbb{R}$ with $|b^*|\leq 6\alpha d^2 + 1$ such that for all $\epsilon \in \left(0,\frac{\alpha}{50d^2}\right)$ and all $f_1,\dots,f_r\in \mathcal{F}$, w.h.p over sampling of $W$, if
$ \left\|\sum_{i=1}^r u_i f_i(Wx) - \left[\inner{w^*,x} + b^*\right]_+\right\| \leq \epsilon $, then
$ r \cdot \max_i  |u_i| \geq (1-576d^4\epsilon^2)\frac{1}{144\alpha d^2}\exp\left( c_3 d\right) $
for a universal constant $c_3$.
\end{remark}

\subsection{General Features and Kernel Methods}

In the previous subsection, we assumed that our features have a structure of the form $x\mapsto f(Wx)$ for a random matrix $W$. 
We now turn to a more general case, where we are given features $x\mapsto f(x)$ of any kind without any assumptions on their internal structure, as long as they are sampled from some fixed distribution. Besides generalizing the setting of the previous subsection, it also captures the setting of Subsection \ref{subsec:optimization with coupling}, where the coupling method is used in order to show that 
\begin{equation}\label{eq:coupling random features}
    [\inner{w^{(t)},x}]_+ \approx \inner{w^{(t)},x}\mathbbm{1}_{\inner{w^{(0)},x}\geq 0}. 
\end{equation}  This is because it is a weighted sum of the random features $f_{i,j}(x)=x_j\mathbbm{1}_{\inner{w_i,x}\geq 0}$ where $w_i$ are initialized by standard Xavier initialization. Moreover, in this general setting, we also capture kernel methods, since kernel predictors can be seen as linear combinations of features of the form $x\mapsto k(x_i,x)$ where $x_i$ are elements of the sampled training data. We show that even in this general setting, it is impossible to approximate single ReLU neurons in the worst-case (at the cost of proving this for some target weight vector $w^*$, instead of any $w^*$). 

\begin{theorem}\label{thm:limitations of random features with coupling}
There exist universal constants $c_1,c_2,c_3$ such that the following holds. Let $d>c_1$, and let $\mathcal{F}$ be a family of functions from $\reals^d$ to $\reals$, such that $\|f\|\leq \exp(c_2 d)$ for all $f\in\mathcal{F}$. Also, for some $r\in\mathbb{N}$, let $D$ be an arbitrary distribution over tuples $(f_1,\ldots f_r)$ of functions from $\Fcal$. Then there exists $w^*\in\mathbb{R}^d$ with $\|w^*\|=d^3$, and $b^*\in\mathbb{R}$ with $|b^*|\leq 6d^4 + 1$, such that with probability at least $1-r\exp(-c_3 d)$ over sampling $f_1,\ldots,f_r$, if
\[ \left\|\sum_{i=1}^r u_if_i(x) - \left[\inner{w^*,x} + b^*\right]_+\right\|^2 \leq \frac{1}{50}~, \]
then
\[ r \cdot \max_i |u_i| \geq \frac{1}{200d^4}\exp\left(c_3 d\right). \]
\end{theorem}

The proof is similar to the proof of Theorem \ref{thm:limitations of random features}. The main difference is that we do not have any assumptions on the distribution of the random features (as the assumption on the distribution on $W$ in Theorem \ref{thm:limitations of random features}), hence we can only show that there \textit{exists} some ReLU neuron that cannot be well-approximated. On the other hand, we have almost no restrictions on the random features, e.g. they can be multi-layered neural network of any architecture and with any random initialization, kernel-based features on sampled training sets, etc.

\subsubsection{Acknowledgements}

This research is supported in part by European Research Council (ERC) grant 754705. We thank Yuanzhi Li for some helpful comments on a previous version of this paper, and for Pritish Kamath and Alex Damian for spotting a bug in a previous version of the paper.

\setcitestyle{numbers}
\bibliographystyle{abbrvnat}
\bibliography{mybib}

\appendix
\addcontentsline{toc}{section}{Appendices}
\section*{Appendices}

\section{Comparison to Previous Works}\label{appendix:comparison to previous works}
\subsection{Optimization with Coupling}\label{appendix:optimization with coupling}
The method from Subsection \ref{subsec:optimization with coupling} is used in several works, for example:

\begin{itemize}
    \item Li and Liang \cite{li2018learning} show a generalization bound for ReLU neural networks where there is a strong separability assumption on the distribution of the data, which is critical in the analysis.
    \item In Du et al. \cite{du2018gradient2}, it is shown that under some assumptions on the data, neural network with ReLU activation would reach a global minimum of the empirical risk in polynomial time. 
    \item Allen-Zhu et al. \cite{allen2018convergence} also show an empirical risk bound on a multi-layer feed-forward neural network with ReLU activation and relies on separability assumption on the data.
    \item Allen-Zhu et al. \cite{allen2018learning} show with almost no assumptions on the distribution of the data that the generalization of neural networks with two or three layers is better then the generalization of a "ground truth" function for a large family of functions, which includes polynomials.
    \item In Allen-Zhu and Li \cite{allen2019can} similar techniques are used to show a generalization bound on recurrent neural networks. 
    \item Cao and Gu \cite{cao2019generalization} show a generalization result for multi-layered networks, where the generalization is compared to a family of functions that take an integral form, similar to the one developed in Theorem \ref{theorem about approximating integral as a sum}. This integral form is also studied in \cite{sun2018random}, \cite{klusowski2018approximation} in the context of studying neural networks as random features.
\end{itemize}
  
All the above fix the weights of the output layer and consider only the ReLU activation function. Note that while using ReLU as an activation function, it is possible to show that fixing the output layer does not change the expressive power of the network. With that said, in practice all layers of neural networks are being optimized, and the optimization process may not work as well if some of the layers are fixed. 

\subsection{Optimization on all the Layers}\label{appendix:optimization on all the layers}
The methods from Subsection \ref{subsec:optimization on all the layers} are also used in several works, for example:

In Andoni et al. \cite{andoni2014learning} this approach is also used to prove that neural networks can approximate polynomials. There the weights are drawn from a complex distribution, thus optimization is done on a complex domain which is non standard. Moreover, that paper uses the exponent activation function and assumes that the distribution on the data is uniform on the complex unit circle.

In Daniely et al. \cite{daniely2016toward} and Daniely \cite{daniely2017sgd} this approach is used to get a generalization bound with respect to a large family of functions, where the network may have more than two layers and a different architecture than simple feed-forward. The methods used there are through the "conjugate kernel", which is a function corresponding to the activation and architecture of the network. The proof of our result is relatively more direct, and gives a bound which correspond directly to the activation function, without going through the conjugate kernel. Moreover, those papers do not quantitatively characterize the class of polynomials learned by the network, with an explicit proof. 

Du and Lee \cite{du2018gradient} rely on the same methods and assumptions as Du et al. \cite{du2018gradient2} to show an empirical risk bound on multi-layer neural network where a large family of activations is considered and with several architectures, including ResNets and convolutional ResNets. However, this does not imply a bound on the population risk.

\section{Proofs from section \ref{section limitations of random features}}\label{appendix proofs from section limitations of random features}

\begin{proof}[Proof of Proposition \ref{prop:warm up linear predictor}]
By definition of our norm and the fact that $\E[xx^\top]=I$ for a standard Gaussian distribution, it is enough to show that 
\[
\min_{u_1\ldots u_r}\norm{\sum_{i=1}^{r}u_i w_i-w^*}^2~\geq~ \frac{1}{4}~.
\]
Re-writing $\sum_{i=1}^{r}u_i w_i$ as $Wu$ (where $W$ is a $d\times r$ matrix and $u$ is a vector in $\reals^d$), the left hand side is equivalent to $\min_u\norm{Wu-w^*}^2$, which is a standard linear least square problem, with a minimum at $u=(W^\top W)^{\dagger}W^\top w^*$, and a minimal value of $\norm{(W(W^\top W)^{\dagger}W^\top-I)w^*}^2$. Letting $W=USV^\top$ be an SVD decomposition of $W$ (where $U,V$ are orthogonal matrices) and simplifying a bit, we get:
\begin{align*}
& \norm{(W(W^\top W)^{\dagger}W^\top-I)w^*}^2 = \norm{(USV^\top (VSU^\top U SV^\top)^\dagger VSU^\top  -UU^\top)w^*}^2 \\
& = \norm{U}^2\cdot \norm{(SV^\top(VS^2V^\top)^\dagger VSU^\top - U^\top)w^*}^2 = \norm{MU^\top w^*}^2,
\end{align*}
where $M$ is (with probability $1$) a fixed $d\times d$ diagonal matrix, with a diagonal composed of $d-r$ ones and $r$ zeros. Moreover, by symmetry, $U^\top w^*$ is uniformly distributed on the unit sphere, so it is enough to understand the distribution of $\norm{Mz}^2$ for a random unit vector $z$. Since the function $z\mapsto \norm{Mz}^2$ is $1$-Lipschitz on the unit sphere, it follows by standard concentration results for Lipschitz functions (see for example \cite{ledoux2001concentration}) that with probability at least $1-\exp(-cd)$ (for some universal constant $c$), $\E[\norm{Mz}^2]-\norm{Mz}^2\leq \frac{1}{4}$. Finally, $\E[\norm{Mz}^2]\geq \frac{1}{2}$, since clearly $\E[\norm{Mz}^2]\geq\E[\norm{(I-M)z}^2]$ (note that $M$ has more ones than zeros on the diagonal, as $d-r\geq r$), yet $\E[\norm{Mz}^2]+\E[\norm{(I-M)z^2}]=\E[\norm{z}^2]=1$. Combining the above, the result follows.

\end{proof}

In the proof of Proposition \ref{proposition properties of psi} we rely on the following claim from \cite[Lemma 5]{shamir2018distribution} \footnote{In order to deduce it we only need to note that since $\psi$ is odd, its  first Fourier coefficient is $a_0=0$, and we can take $g = \widehat{f\phi}$ and use the fact that Fourier transform preserves inner products and norms.}:
\begin{claim}\label{claim from ohad's paper about periodic functions}
For any $f\in L^2(\mathbb{R}^d)$, and odd periodic function $\psi:\mathbb{R}\rightarrow \mathbb{R}$ if $d > c'$ (for some universal constant c'), and we sample $w$ uniformly from $\{w:\ \|w\|=2r\}$, it holds that:
\[ \mathbb{E}_w\Big(\inner{f(x),\psi\big(\inner{w,x}\big)}^2\Big) \leq 10 \|f\|\cdot (\exp(-cd) + \sum_{n=1}^\infty \exp(-nr^2))\]
for a universal constant $c>0$.
\end{claim}
\begin{proof}[Proof of Proposition \ref{proposition properties of psi}]
For every $x_0\in[-a,a-4]$ we have that:
\begin{align*}
\psi(x_0 + 4) & = [x_0 + 4 + a]_+ + \sum_{n=1}^a 2[x_0 + 4 + a - 2n]_+(-1)^n -1 \\
& = x_0 + a + 4 + \sum_{n=1}^{\left\lfloor\frac{x_0 + a}{2}\right\rfloor + 2} 2(x_0 + 4 + a - 2n)(-1)^n -1\\
& = x_0 + a + \sum_{n=1}^{\left\lfloor\frac{ x_0 + a }{2}\right\rfloor} 2(x_0 + a - 2n) - 1 + 4 + \\ & + \sum_{n=1}^{\left\lfloor\frac{x_0 + a}{2}\right\rfloor} 8(-1)^n + (-1)^{\left\lfloor\frac{ x_0 + a }{2}\right\rfloor + 1}2(x_0 + 4 + a - (x_0+a+2)) +\\ &+ (-1)^{\left\lfloor\frac{ x_0 + a }{2}\right\rfloor + 2}2(x_0 + 4 + a - (x_0+a+4))\\
 & = [x_0 + a]_+ + \sum_{n=1}^a 2[x_0 + a - 2n]_+ - 1 = \psi(x_0),
\end{align*}
where we used the fact that $a$ is odd. This proves that $\psi(x)$ is periodic in $[-a,a]$ with a period of $4$. To prove that $\psi(x)$ is an odd function, note that:
\[ \psi(-2)=\psi(2)= \psi(0) = 0 \]
and also that $\psi(1) = -\psi(-1) = \pm 1$, and between every two integers $\psi(x)$ is a linear function. The exact value of $\psi(1)$ and $\psi(-1)$ depends on whether $\frac{a+1}{2}$ is even or odd. Thus, $\psi(x)$ is an odd function in the interval $[-2,2]$, and because it is periodic with a period of $4$, it is also an odd function in the interval $[-a,a]$. This proves item \ref{first item of psi theorem} in the proposition. \\
For item \ref{second item of psi theorem}, using the fact that $x\in \mathbb{R}^d$ is distributed symmetrically we can assume w.l.o.g that $w^* = \begin{pmatrix} d \\ 0 \\ \vdots \\ 0 \end{pmatrix}$. Thus,
\begin{align*}\label{norm of psi}
\|\psi(\inner{w^*,x})\|^2 & = c_d\int_{x\in \mathbb{R}^d} |\psi\big(\inner{w^*,x}\big)|^2 e^{\frac{-\|x\|^2}{2}}dx \nonumber\\
& = c_d\int_{-\infty}^\infty |\psi(dx_1)|^2e^{\frac{-x_1^2}{2}}dx_1\cdot \int_\infty^\infty e^{\frac{-x_2^2}{2}}dx_2\cdots \int_\infty^\infty e^{\frac{-x_d^2}{2}}dx_d \\
& \geq  \frac{1}{d\sqrt{2\pi}}\int_{-\infty}^\infty |\psi(x_1)|^2e^{-\left(\frac{x_1}{d}\right)^2}dx_1 \nonumber \geq \frac{1}{d\sqrt{2\pi}}\int_{-d}^d |\psi(x_1)|^2e^{-1}dx_1 \\
& \geq \frac{1}{d\sqrt{2\pi}}d e^{-1}\cdot  \frac{4}{3} \geq \frac{4}{3e\sqrt{2\pi}}\geq  \frac{1}{6}.
\end{align*} 
In the above, we used the fact that for every interval of the form $[n,n+2]$ for $n\in [-a,a-2]$, the integral 
\[ \int_{[n,n+2]} |\psi(x_1)|^2dx_1 = \frac{4}{3}, \]
and there are $d$ such intervals in $[-d,d]$. \\
For item \ref{third item of psi theorem}, define $\widetilde{\psi}(x)$ to be equal to $\psi(x)$ on $[-a,a]$, and continue it to $[-\infty,\infty]$ such that it is periodic. 
Let $g(x) = \widetilde{\psi}(x) - \psi(x)$, then $g(x) =0 $ for $x\in [-a,a]$ and $|g(x)| \leq x$ for $|x|>a$. Now for every $w\in \mathbb{R}^d$ with $\|w\| = d$ we get that:
\begin{align}
    \|g(\inner{w,x})\|^2 & = \mathbb{E}_x\left[g^2(\inner{w,x})\right] \leq \mathbb{E}_x\left[\mathbbm{1}_{|\inner{w,x}| \geq a}\inner{w,x}^2\right] \nonumber\\
    & \leq \sqrt{\mathbb{E}_x\left[\mathbbm{1}_{|\inner{w,x}| \geq a}^2\right]}\cdot \sqrt{\mathbb{E}_x\left[\inner{w,x}^4\right]} \nonumber\\
    & \leq \sqrt{\mathbb{E}_x\left[\mathbbm{1}_{|\inner{w,x}| \geq a}\right]} \cdot \|w\|^2\sqrt{\mathbb{E}_x\left[\inner{\frac{w}{\|w\|},x}^4\right]}\\
    & = 3d^2\sqrt{P\left(\left|\inner{w,x}\right| \geq a\right)} \label{bound on the norm of g with input of inner w x}
\end{align}
where we used the fact that since $x\sim N(0,I_d)$ then $\inner{\frac{w}{\|w\|},x}$ has a standard Gaussian distribution, hence its fourth moment is $3$. Also $\inner{w,x}\sim N(0,d)$, Hence 
\[P\left(|\inner{w,x}| \geq 6d^2 + 1\right) \leq \exp(-d), \]
and using \eqref{bound on the norm of g with input of inner w x} we get:
\begin{equation}\label{eq:bound on g(w,x) with exponent}
    \|g(\inner{w,x})\|^2 \leq 3d^2\exp(-d) \leq \exp(-cd),
\end{equation}
for a constant $c$. Now for every $f\in L^2(\mathbb{R}^d)$:
\begin{equation}\label{bound on inner product between f and psi}
\mathbb{E}_w\left(\inner{f(x),\psi\big(\inner{w,x}\big)}^2\right) \leq 2\mathbb{E}_w\left(\inner{f(x),\widetilde{\psi}\left(\inner{w,x}\right)}^2\right) + 2\mathbb{E}_w\left(\inner{f(x),\widetilde{\psi}\left(\inner{w,x}\right) - \psi\left(\inner{w,x}\right)}^2\right).  
\end{equation}
Using Cauchy-Schwartz and \eqref{eq:bound on g(w,x) with exponent} we can bound the second term:
\[ \mathbb{E}_w\left(\inner{f(x),\widetilde{\psi}\left(\inner{w,x}\right) - \psi\left(\inner{w,x}\right)}^2\right) = \mathbb{E}_w\left(\inner{f(x),g(\inner{w,x}}^2\right) \leq \|f\|^2\exp(-cd), \]
and finally using Claim \ref{claim from ohad's paper about periodic functions} on $\widetilde{\psi}(x)$, and taking $r=d$ we can bound the first term of \eqref{bound on inner product between f and psi}, by changing the constant from the claim by a factor of at most $4$. Thus, there exists a universal constant $c$ such that:
\[
\mathbb{E}_w\left(\inner{f(x),\psi\big(\inner{w,x}\big)}^2\right) \leq \|f\|^2 \exp(-cd)
\]
\end{proof}

\begin{proof}[Proof of Theorem \ref{thm:limitations of random features with coupling}]
Take $\psi(x)$ from Proposition \ref{proposition properties of psi} and denote for $w\in \mathbb{R}^d$, $\psi_w(x) = \psi(\inner{w,x})$. If we sample $w^*$ uniformly from $\{w:\ \|w\|=d\}$ and $(f_1,\dots, f_r)\sim D$ then:
\begin{align}
    &\mathbb{E}_{w^*}\left[ \mathbb{E}_{(f_1,\dots,f_r)}\left[\left|\inner{\sum_{i=1}^r f_i,\psi_{w^*}}\right| \right]\right] = \mathbb{E}_{(f_1,\dots,f_r)}\left[ \mathbb{E}_{w^*}\left[\left|\inner{\sum_{i=1}^r f_i,\psi_{w^*}}\right| \right]\right] \nonumber\\ 
    & \leq \mathbb{E}_{(f_1,\dots,f_r)}\left[ 20 \left\|\sum_{i=1}^r f_i\right\| \exp(-cd)\right] \leq \mathbb{E}_{(f_1,\dots,f_r)}\left[ \sum_{i=1}^r 20\|f_i\| \exp(-cd)\right] \leq r\exp(-c_3d) \nonumber
\end{align} 
where $c_3$ is a universal constant that depends only on the constant $c$ from Proposition \ref{proposition properties of psi} and on $c_2$. Thus, there exists $w^*$ such that:
\begin{equation}\label{eq:oblivious features w*}
\mathbb{E}_{(f_1,\dots,f_r)}\left[\left|\inner{\sum_{i=1}^r f_i,\psi_{w^*}}\right| \right] \leq r\exp(-c_3d) .    
\end{equation}
Using Markov's inequality on \eqref{eq:oblivious features w*}, with the fixed $w^*$ that was found and dividing $c_3$ by a factor of $2$, we get w.p $>1-r\exp\left(-c_3d\right)$ over sampling of $(f_1,\dots,f_r)\sim D$ that:
\[ 
\left|\inner{\sum_{i=1}^r f_i,\psi_{w^*}}\right|\leq r\exp(-c_3d) 
\]
The rest of the proof is the same as the proof of Theorem \ref{thm:limitations of random features}, except for fixing the $w^*$ we found above.
\end{proof}

\section{ Approximating Polynomials Using Expectation of Random Features}\label{appendix about representing polynomials in an integral form}
We will use the Legendre polynomials in the one variable and multi-variable case. Let $p_1(w),\dots,p_k(w)$ be the one variable Legendre polynomials, these polynomials are an orthogonal basis for one variable polynomials with respect to the inner product: 
\[ \inner{f,g} = \int_{-1}^1 f(w)g(w)dw. \]
They are normalized by $p_0(w) = 1$, and their inner product is:
\[ \inner{p_i,p_j} = \delta_{i,j} \frac{2}{2i + 1}, \]
where $\delta_{i,j} = 1$ if $i=j$ and $0$ otherwise. Let $J=(j_1,\dots,j_d)$ be a multi index, and define for $w\in \mathbb{R}^d$: 
\[ p_J(w) = p_{j_1}(w_1)\cdots p_{j_d}(w_d). \]
Using tensor product of polynomial spaces, the polynomials $p_J$ for all multi-indices $J$ form an orthogonal base for $d$-dimensional polynomials (see \cite{gelbaum1961bases}), with respect to the inner product:
\[ \inner{f,g} = \int_{w\in [-1,1]^d}f(w)g(w) dw. \]

\begin{proof}[Proof of Theorem \ref{theorem about integral representation of polynomials using legendre}]
Let $E$ and $F$ be the change of basis operators from monomials to Legendre polynomials, and vice versa (in dimension $d$). That is, given some multi-index $J$, there are coefficients $e_{J,J'},f_{J,J'}\in\reals$ for every $|J'|\leq |J|$ such that:
\begin{equation}\label{eq:operators E F}
w^J = \sum_{|J'|\leq |J|}e_{J,J'}p_{J'}(w),~~~~~~ p_J(w) = \sum_{|J'|\leq |J|}f_{J,J'}w^{J'}    
\end{equation}
We define $n := \max\left\{\tilde{n}, k\log(16d) + \log\left(\frac{\alpha}{a}\right) + \log\left(\epsilon^{-1}\right)\right\}$, and we will construct:
\[
g(w) = \sum_{|J|\leq n}c_Jp_J\left(\sqrt{d}w\right)
\]
which satisfies the requirements of the theorem, where the coefficients $c_J$ will be determined later. We will first show how to choose the coefficients $c_J$ so that $g(w)$ can indeed approximate polynomials, and then we will prove items (1) and (2) by bounding $\max_{w\in \left[-\frac{1}{\sqrt{d}},\frac{1}{\sqrt{d}}\right]} |g(w)|$.
Since $\sigma$ is analytic we can use its Taylor expansion to get:

\begin{align}
    c_d  \int_{w\in \left[\frac{-1}{\sqrt{d}},\frac{1}{\sqrt{d}}\right]^d}\sigma(\inner{w,x})g(w)dw &= c_d \int_{w\in \left[\frac{-1}{\sqrt{d}},\frac{1}{\sqrt{d}}\right]^d}\sum_{i=0}^\infty a_i \inner{w,x}^i g(w)dw \nonumber\\
    & = c_d \sum_{i=0}^\infty a_i \int_{w\in \left[\frac{-1}{\sqrt{d}},\frac{1}{\sqrt{d}}\right]^d} \inner{w,x}^i g(w)dw \label{main equation of the theorem after using taylor expansion}
\end{align}

Plugging $g(w)$ and $c_d$ into \eqref{main equation of the theorem after using taylor expansion} and using the change of variables $\sqrt{d}w\mapsto w$ gives:
\begin{align}
&\left(\frac{\sqrt{d}}{2}\right)^d\sum_{i=0}^\infty a_i \int_{w\in \left[\frac{-1}{\sqrt{d}},\frac{1}{\sqrt{d}}\right]^d}\inner{w,x}^i g(w)dw \nonumber \\
&= \left(\frac{\sqrt{d}}{2}\right)^d\sum_{i=0}^\infty a_i \int_{w\in \left[\frac{-1}{\sqrt{d}},\frac{1}{\sqrt{d}}\right]^d}\inner{w,x}^i \sum_{|J'|\leq n} c_{J'} p_{J'}\left(\sqrt{d}w\right)dw \nonumber \\
& = \left(\frac{1}{2}\right)^d\sum_{i=0}^\infty \frac{a_i}{\left(\sqrt{d}\right)^i} \int_{w\in [-1,1]^d}\inner{w,x}^i \sum_{|J'|\leq n} c_{J'} p_{J'}(w)dw \nonumber \\
& = \left(\frac{1}{2}\right)^d\sum_{i=0}^\infty \sum_{|J|=i}\frac{a_i}{\left(\sqrt{d}\right)^i}x^J \int_{[-1,1]^d}w^J \sum_{|J'|\leq n} c_{J'} p_{J'}(w)dw
. \label{integral representation after change of variables sqrt of d w}
\end{align}
Using \eqref{eq:operators E F} to expand $w^J$ in the Legendre basis, the above is equal to:

\begin{align}
& \left(\frac{1}{2}\right)^d\sum_{i=0}^\infty \sum_{|J| = i}\frac{a_i}{\left(\sqrt{d}\right)^i}x^J \int_{[-1,1]^d} \sum_{J''\leq J} e_{J,J''} p_{J''}(w) \sum_{|J'|\leq k} c_{J'} p_{J'}(w)dw \nonumber \\
& = \sum_{i=0}^\infty \sum_{|J| = i}x^J \left(\frac{1}{2}\right)^d \frac{a_i}{\left(\sqrt{d}\right)^i} \sum_{J'\leq J} c_{J'} e_{J,J'} \int_{w\in [-1,1]^d} p_{J'}^2(w) dw \label{representation of the integral form as monomials in x J}
\end{align}
where in the last equality we used the orthogonality of the Legendre polynomials. To get that the above equation will approximate the polynomial $P(x)= \sum_{|J|\leq k}\alpha_J x^J$, we need that for every monomial $x^J$, the coefficient $\alpha_J$ will be equal to the coefficient of $x^J$ from \eqref{representation of the integral form as monomials in x J} for every $|J|\leq k$ and equal to $0$ for every $|J|>k$. Denote $\|p_{J}\|^2 = \int_{w\in [-1,1]^d} p_{J}^2(w) dw$, let $\tilde{c}_J:= c_J\norm{p_J}^2$ and also denote:
\[
y_J = \begin{cases} \frac{\alpha_J}{a_J}\cdot 2^d\left(\sqrt{d}\right)^i, & |J|\leq k \\
0, & |J| > k
\end{cases}
\]
We define the coefficients $c_J$ such that $E\tilde{c}_J = y_J$ for every $|J|\leq n$. By this definition of $c_J$ and \eqref{representation of the integral form as monomials in x J} we get that:
\begin{align}
&c_d  \int_{w\in \left[\frac{-1}{\sqrt{d}},\frac{1}{\sqrt{d}}\right]^d}\sigma(\inner{w,x})g(w)dw \nonumber\\
& =   c_d \sum_{i=0}^n a_i \int_{w\in \left[\frac{-1}{\sqrt{d}},\frac{1}{\sqrt{d}}\right]^d} \inner{w,x}^i g(w)dw +    c_d \sum_{i=n+1}^\infty a_i \int_{w\in \left[\frac{-1}{\sqrt{d}},\frac{1}{\sqrt{d}}\right]^d} \inner{w,x}^i g(w)dw \nonumber \\
& = P(x) +  c_d \sum_{i=n+1}^\infty a_i \int_{w\in \left[\frac{-1}{\sqrt{d}},\frac{1}{\sqrt{d}}\right]^d} \inner{w,x}^i g(w)dw \nonumber
\end{align}
If the above equation holds, then:
\begin{align}\label{eq:item 1 bound}
    \max_{\norm{x}\leq 1}\left|c_d\int_{w\in \left[-\frac{1}{\sqrt{d}},\frac{1}{\sqrt{d}}\right]}\sigma(\inner{w,x}) g(w)dw - P(x)\right| =&  \max_{\norm{x}\leq 1}\left|c_d \sum_{i=n+1}^\infty a_i \int_{w\in \left[\frac{-1}{\sqrt{d}},\frac{1}{\sqrt{d}}\right]^d} \inner{w,x}^i g(w)dw\right| \nonumber\\
    \leq  & \tilde{g}\sum_{i=n+1}^\infty a_i c_d\int_{w\in \left[\frac{-1}{\sqrt{d}},\frac{1}{\sqrt{d}}\right]^d} 1\cdot dw \nonumber\\
    \leq & \tilde{g}\sum_{i=n+1}^\infty \left(40d^{1.5}\right)^{-i} \leq \tilde{g}\left(40d^{1.5}\right)^{-n}
\end{align}
where $\tilde{g} = \max_{w\in \left[-\frac{1}{\sqrt{d}},\frac{1}{\sqrt{d}}\right]} |g(w)|$, we used that $\inner{w,x}^i \leq \norm{w}^i\norm{x}^i \leq 1$ by our assumption on $x$, and the domain of the integral, and we also used that $n\geq \tilde{n}$. We are left with bounding $\tilde{g}$, which will also bound \eqref{eq:item 1 bound}.

Recall that we defined $c_J$ such that $E\tilde{c}_J = y_J$. By applying the operator $F$ on both sides, we get that $\tilde{c}_J = Fy_J$. Using the definitions of $\tilde{c}_j$ and $y_J$ we get the following definition of $c_J$:
\begin{equation}\label{eq:c J}
    c_J = \frac{2^d\left(\sqrt{d}\right)^i}{\norm{p_J}^2} \cdot\sum_{|J'|\leq k}f_{J,J'}\frac{\alpha_{J'}}{a_{J'}}
\end{equation}

We now turn to bound $|c_J|$. For this, we will first need to bound $f_{J,J'}$ and $\frac{1}{\norm{p_J}^2}$. Using Rodrigues' formula we can derive the following explicit representation of the $i$-th Legendre polynomial:
\[
p_i(x) = 2^i\sum_{j=1}^i x^j\binom{i}{j}\binom{\frac{i+j-1}{2}}{i}
\]
where the second binomial term is given by using the generalized form of the binomial coefficient. Given a multi-index $J=(i_1,\dots,i_d)$ with $|J|=\ell$ we use the above to write a multivariate Legendre polynomial as:
\begin{align}\label{eq:legendre as monomials}
p_J(x) &= p_{i_1}(x_1)\cdots p_{i_d}(x_d) \nonumber\\
&= \left(2^{i_1}\sum_{j_1=1}^{i_1} x^{j_1}\binom{i_1}{j_1}\binom{\frac{i_1+j_1-1}{2}}{i_1}\right)\cdots \left(2^{i_d}\sum_{j_d=1}^{i_d} x^{j_d}\binom{i_d}{j_d}\binom{\frac{i_d+j_d-1}{2}}{i_d}\right)~.
\end{align}
For a multi-index $J'\leq J$, the term $f_{J,J'}$ is equal to the coefficient of the monomial $x^{J'}$ in \eqref{eq:legendre as monomials}. Each such coefficient is a multiplication of $d$ terms. Denoting $J'=(\ell_1,\dots,\ell_d)$, and using the bound $\left|\binom{i}{j}\right|\leq 2^i$, we can bound $f_{J,J'}$ as:
\begin{align}\label{eq:bound on F J J'}
    |f_{J,J'}| &= \left|\left(2^{i_1} \binom{i_1}{\ell_1}\binom{\frac{i_1+\ell_1-1}{2}}{i_1}\right)\cdots \left(2^{i_d} \binom{i_d}{\ell_d}\binom{\frac{i_d+\ell_d-1}{2}}{i_d}\right)\right|\nonumber\\
    &\leq 8^{i_1}\cdots 8^{i_d} = 8^{\ell}
\end{align}
Next, we bound $\frac{1}{\norm{p_J}^2}$. The norm of a single variable Legendre polynomial is given by:
\[
\int_{-1}^1 p_\ell(w)^2dw = \frac{2}{2\ell+1}
\]
Given a multi-index $J=(j_1,\dots,j_d)$ with $|J|=i$ and using this, we can write:
\begin{align*}
    \norm{p_J(x)}^2 &= \int_{w\in[-1,1]^d}p_J^2(w)dw = \left(\int_{-1}^1p^2_{j_1}(w_1)dw_1\right) \cdots \left(\int_{-1}^1p^2_{j_d}(w_d)dw_d\right) \\
    & = \frac{2}{2j_1+1}\cdots \frac{2}{2j_d + 1} 
\end{align*}
Hence, we can bound: 
\begin{equation}\label{eq:bound 1/ norm p J}
\frac{1}{\norm{p_J(x)}^2} \leq \frac{(2j_1 +1)\cdots (2j_d+1) }{2^d}~.    
\end{equation}

If $d\geq i$, then it is clear that the term in the nominator above is largest for $J$ with $i$ indices equals to 1, and the rest 0 (e.g $J=(\underbrace{1,\dots,1}_\text{i},\underbrace{0,\dots,0}_\text{d-i})$). For this case we bound:
\[
\frac{1}{\norm{p_J(x)}^2} \leq \frac{3^i}{2^d}~.
\]
If $d <i$, then the nominator of \eqref{eq:bound 1/ norm p J} is largest when the indices of J are equal, i.e. $J=\left(\left\lceil\frac{i}{d}\right\rceil,\dots,\left\lceil\frac{i}{d}\right\rceil\right)$. In this case, we can bound:
\begin{align*}
    \frac{1}{\norm{p_J(x)}^2} \leq \frac{\left(2\left\lceil\frac{i}{d}\right\rceil + 1\right)^d}{2^d} \leq \frac{\left(2\cdot \frac{i}{d} + 3\right)^d}{2^d} \leq \frac{\left(5\cdot \frac{i}{d}  \right)^d}{2^d}
\end{align*}
Note that for any value of $d\in\{1,\dots,i\}$ we have that $\left(\frac{i}{d}\right)^d\leq 2^i$, hence, we can bound the above as:
\begin{equation}\label{eq:bound  norm p J}
\frac{1}{\norm{p_J(x)}^2} \leq \frac{5^d\cdot 2^i}{2^d} \leq \frac{10^i}{2^d}~,    
\end{equation}
and by the arguments above, this gives us an upper bound also for the case of $d\geq i$~.

Using \eqref{eq:bound on F J J'} and \eqref{eq:bound  norm p J} we are now ready to bound $|c_J|$. By \eqref{eq:c J}, for a multi-index $J$ with $|J|\leq i$ we have that:
\begin{align}\label{eq:bound c J}
    |c_J|  &= \left|\frac{2^d\left(\sqrt{d}\right)^i}{\norm{p_J}^2} \cdot\sum_{|J'|\leq k}f_{J,J'}\frac{\alpha_{J'}}{a_{J'}} \right| \nonumber\\
    &\leq \frac{2^d\left(\sqrt{d}\right)^i}{\norm{p_J}^2} \cdot\sum_{|J'|\leq k} |f_{J,J'}|\cdot \left|\frac{\alpha_{J'}}{a_{J'}}\right| \nonumber\\
    &\leq \left(10\sqrt{d}\right)^i\sum_{|J'|\leq k} 8^k \frac{\alpha}{a}\nonumber\\
    &\leq \left(10\sqrt{d}\right)^i 8^k \frac{\alpha}{a}\sum_{j=0}^k \binom{d}{j} \nonumber\\
    & \leq \left(10\sqrt{d}\right)^i (8(d+1))^k \frac{\alpha}{a}
\end{align}
where we used the bound $\sum_{j=0}^k\binom{d}{j}\leq (d+1)^k$.

Now we turn to bound $\tilde{g}$. For the Legendre polynomial, we have that $\max_{w\in[-1,1]}|p_i(w)|\leq 1$ for every $i$. For the multivariate case: \[
\max_{w\in \left[-\frac{1}{\sqrt{d}},\frac{1}{\sqrt{d}}\right]} |p_J(w)| = \max_{w\in \left[-\frac{1}{\sqrt{d}},\frac{1}{\sqrt{d}}\right]} |p_{j_1}(w_1)\cdots p_{j_d}(w_d)| \leq 1
\]
Using this bound and \eqref{eq:bound c J} we have:
\begin{align}\label{eq:bound tilde g}
    \tilde{g} &= \max_{w\in \left[-\frac{1}{\sqrt{d}},\frac{1}{\sqrt{d}}\right]} |g(w)| \nonumber\\
    &= \max_{w\in \left[-\frac{1}{\sqrt{d}},\frac{1}{\sqrt{d}}\right]} \left| \sum_{|J|\leq n}c_J p_J\left(\sqrt{d}w\right)\right| \nonumber \\
    & \leq \max_{w\in \left[-\frac{1}{\sqrt{d}},\frac{1}{\sqrt{d}}\right]}  \sum_{|J|\leq n}|c_J| \left|p_J\left(\sqrt{d}w\right)\right| \nonumber\\
    & \leq \left(10\sqrt{d}\right)^n (8(d+1))^k \frac{\alpha}{a} \sum_{i=0}^n \binom{d}{i} \nonumber\\
    &\leq \left(10\sqrt{d}\right)^n (8(d+1))^k \frac{\alpha}{a}\cdot (d+1)^n \leq \left(20d^{1.5}\right)^n(16d)^k\frac{\alpha}{a}
\end{align}
Plugging in  $n =\max\left\{\tilde{n},  k\log(16d) + \log\left(\frac{\alpha}{a}\right) + \log\left(\epsilon^{-1}\right)\right\}$, and our bound of $\tilde{g}$ into \eqref{eq:item 1 bound} we get that:
\begin{align*}
&\max_{\norm{x}\leq 1}\left|c_d\int_{w\in \left[-\frac{1}{\sqrt{d}},\frac{1}{\sqrt{d}}\right]}\sigma(\inner{w,x}) g(w)dw - P(x)\right| \\
&\leq \left(20d^{1.5}\right)^n(16d)^k\frac{\alpha}{a}\cdot \left(\frac{1}{40d^{1.5}}\right)^n \\
&= \left(\frac{1}{2}\right)^n(16d)^k\frac{\alpha}{a} \leq \epsilon
\end{align*}
which prove item (1). Plugging in $n$ into \eqref{eq:bound tilde g} gives us:
\begin{align*}
    \tilde{g} &\leq \left(20d^{1.5}\right)^{\max\{\tilde{n},k\log(16d) + \log\left(\frac{\alpha}{a}\right) + \log\left(\epsilon^{-1}\right)\}}(16d)^k\frac{\alpha}{a} \\
    &\leq \left(320d^{2.5}\right)^{\max\{\tilde{n},k\}}\cdot (16d)^k\left(\frac{\alpha}{a}\right)^2\epsilon^{-1}\\
    & = \left(5120d^{3.5}\right)^{\max\{\tilde{n},k\}}\left(\frac{\alpha}{a}\right)^2\epsilon^{-1}
\end{align*}
which proves item (2).

\end{proof}

\section{Random Features Concentrate Around their Expectation}\label{appendix approximating integral to a discrete sum}

\begin{proof}[Proof of Theorem \ref{theorem about approximating integral as a sum}]
    Define $u_i = \frac{g(w_i)}{r}$, and
  \[ \hat{f}(w_1,\dots,w_r,x) = \hat{f}(x) = \sum_{i=1}^{r}u_i \sigma(\inner{w_i,x}). \]
  Observe that $\mathbb{E}_w\left[\hat{f}(x)\right] = f(x)$ and $|u_i| \leq \frac{C}{r}$. Define:
  \[ h(w_1,\dots,w_r,x) = h(x) = \sup_x\left|\hat{f}(w_1,\dots,w_r,x) - \mathbb{E}_w\left[\hat{f}(w_1,\dots,w_r,x)\right]\right| \]
  We will use McDiarmid's inequality to bound $h$. For every $1\leq i \leq r$ and every $\widetilde{w_i}$ with $\|\widetilde{w_i}\|\leq 1$ we have that:
  \begin{align*}
   & \left|h(w_1,\dots,w_r,x) - h(w_1,\dots,w_{i-1},\widetilde{w_i},w_{i+1},\dots,w_r,x\right| \leq  \\
   & \leq \sup_x\left|\frac{g(w_i)\sigma(\inner{w_i,x})}{r} - \frac{g(\widetilde{w_i})\sigma(\inner{\widetilde{w_i},x})}{r}\right|  \leq \frac{2LC}{r}
  \end{align*}
  We will now bound the expectation of $h(x)$. Using \cite[Lemma 26.2]{shalev2014understanding} which bounds $\mathbb{E}[h(x)]$ using Rademacher complexity, where the roles of $x$ and $w$ are switched:
  \begin{align*}
    \mathbb{E}[h(x)] = \mathbb{E}\left[ \sup_x \left| \hat{f}(x)-f(x)\right| \right] \leq \frac{2}{r}\mathbb{E}_{w,\xi} \left[\sup_x \left|\sum_{i=1}^{r}\xi_i u_i\sigma(\inner{w_i,x)}\right|\right]
  \end{align*}
  Where $\xi_1,\dots,\xi_r$ are independent Rademacher random variables (where we write them as $\xi$ for short). Define $\sigma'(x) = \sigma(x) - \alpha$, where $\sigma(0) = \alpha$, then we have that $\sigma'(0)=0$. We use the fact that for i.i.d Rademacher random variables $\xi_1,\dots,\xi_r$ :
\[ \mathbb{E}_\xi\left[\left|\sum_{i=1}^r \xi_i\right|\right] \leq \ \sqrt{r}, \]
combined with \cite[Theorem 12(4)]{bartlett2002rademacher}, Cauchy-Schwartz theorem and our assumptions that $\|x\|,\|w\|\leq 1$ to get:
  \begin{align*}
  \frac{2}{r}\mathbb{E}_{w,\xi} \left[\sup_x \left|\sum_{i=1}^{r}\xi_i u_i\sigma(\inner{w_i,x})\right|\right] &\leq
  \frac{2}{r}\mathbb{E}_{w,\xi}\left[ \sup_x \left|\sum_{i=1}^{r}\xi_i u_i\sigma'(\inner{w_i,x})\right| + \alpha\left|\sum_{i=1}^{r}\xi_i u_i\right|\right] \\
  & \leq \frac{2LC}{r}\mathbb{E}_{w,\xi}\left[\sup_x \left|\sum_{i=1}^{r}\xi_i \inner{w_i,x}\right|\right] + \frac{2C}{r} \mathbb{E}_\xi \left[\alpha\left|\sum_{i=1}^{r}\xi_i\right|\right] \\
  & \leq \frac{2LC}{r}  \mathbb{E}_\xi \left[ \left|\sum_{i=1}^r \xi_i\right| \right] + \frac{2LC}{r}\mathbb{E}_\xi \left[ \left|\sum_{i=1}^r \xi_i\right| \right] \\
  & \leq \frac{4LC}{\sqrt{r}}
  \end{align*}
In total we have:
\[ \mathbb{E}[h(x)] \leq \frac{4LC}{\sqrt{r}}\]
We can now use McDiarmid's inequality on $h(x)$ to get that:
  \begin{equation}
    P\left(h(x) - \frac{4LC}{\sqrt{r}} \geq \epsilon\right) \leq P\left(h(x) - \mathbb{E}_w(h(x)\right) \geq \epsilon) \leq \exp\left(-\frac{r \epsilon^2}{4L^2C^2}\right)
  \end{equation}
  Replacing the right hand side with $\delta$ we get that w.p $>$ $1-\delta$:
   \[ \sup_x \left|\hat{f}(x) - f(x) \right| \leq \frac{LC}{\sqrt{r}}\left(4 + \sqrt{2 \log\left(\frac{1}{\delta}\right)}\right) \]
\end{proof}

\section{ SGD on Over-Parameterized Networks Competes with Random Features}\label{appendix optimization with respect to a target matrix}

\begin{lemma} \label{bounding the norm of W 0 and U 0 during SGD}
Let $\|W_0\|,\ \|U_0\| \leq B$ with $B \geq 2$, then for every $\epsilon > 0$ if we run SGD with learning rate of $\eta = \frac{\epsilon}{LB^2}$ we have that for all $t\leq \frac{B}{2\epsilon}$:
\begin{enumerate}
  \item $\|W_t\|, \|U_t\|\leq B + 1$
  \item $\left\|\sigma(W_tx)-\sigma(W_0x)\right\| \leq  2L  t\epsilon $
\end{enumerate}
\end{lemma}

\begin{proof}
  We prove the first part by induction on $t$. First trivially it is true for $t=0$. Assume it is true for all $t\leq \frac{B}{2\epsilon}$. The gradients of  $L_D(U,W)$ are:
  \begin{align}
&\frac{\partial l\left(N(W,U,x),y\right)}{\partial W} = \mathbbm{1}_{\left(1-y\cdot N(U,W,x)\geq 0 \right)}(x\tilde{U})^T \label{gradient with respect to W}\\
&\frac{\partial l(N(W,U,x),y)}{\partial U} = \mathbbm{1}_{(1-y\cdot N(U,W,x))}\sigma(Wx) \label{gradient with respect to U}
\end{align}
Here $\tilde{U}_i = u_i\cdot \sigma ' (\inner{w_i,x})$, and we look at $x$ as a matrix in $\mathbb{R}^{d\times 1}$ hence $x\tilde{U}\in \mathbb{R}^{d\times r}$. 

We bound the gradients of $L_D(U,W)$ using \eqref{gradient with respect to U} and \eqref{gradient with respect to W}, the assumptions on $\sigma$ and that $\|x\|\leq 1$:
 \begin{align*}
 &\left\| \frac{\partial l\left(N(W,U,x),y\right)}{\partial W} \right\| \leq L\|U\| \\
 &\left\| \frac{\partial l\left(N(W,U,x),y\right)}{\partial U}  \right\| \leq \left\|\sigma(Wx)\right\| \leq \|\sigma(Wx)-\sigma(0)\| + \|\sigma(0)\| \leq L + L\|W\|
 \end{align*}
  
Using the bounds on the gradient, at each step of SGD the norm of $W_{t+1}$ changed from the norm of $W_t$ by at most $\eta L\|U_t\|$. Thus, after $t$ iterations we get that
  \begin{equation*}
  \|W_{t+1}\| \leq \|W_0\| + \sum_{i=1}^{t}\eta L\|U_i\|\leq B + t\eta L(B+1) \leq B + \frac{\epsilon}{B^2}\frac{B}{2\epsilon}(B+1) \leq B + 1.
  \end{equation*}
In the same manner for $U_{t+1}$:
\begin{align*}
 \|U_{t+1}\| &\leq \|U_0\| + \sum_{i=1}^t \eta (L+L\|W_i\|) \leq B + t \eta L(B+1) + t\eta L \\
& \leq B+ \frac{\epsilon}{B^2}\frac{B}{2\epsilon}(B+1) + \frac{1}{2B} \leq B + 1.
\end{align*}
  For the second part, using the previous part we get that:
  \[ \|W_{t+1} - W_0\| \leq \sum_{i=1}^{t}\eta \|U_i\| \leq t\eta (B+1) = t\epsilon \frac{B+1}{B}. \]
  Now we use the fact that $\sigma$ is $L$-Lipschitz, $\|x\|\leq 1$ and $|B| \geq 1 $ to get that:
  \[ \|\sigma(W_tx)-\sigma(W_0x)\| \leq 2L  t\epsilon. \]
\end{proof}

We will also use the following theorem about convex online learning (see \cite[Theorem 21.15]{shalev2014understanding}):

\begin{theorem}\label{online learning theorem from SSS}
Let $f_1,\dots, f_T:\mathbb{R}^d\rightarrow \mathbb{R}$ be L-Lipschitz convex functions. Assume that $x_{t+1} = x_t - \eta \nabla f_t(x_t)$, then for any $x^*\in\mathbb{R}^d$ we have that:
\[ \sum_{t=1}^{T}f_t(x_t) \leq \sum_{t=1}^{T}f_t(x^*)+\frac{\|x^* - x_0\|^2}{2\eta} + \frac{\eta T L^2}{2} \]
\end{theorem}

Now we are ready to prove the generalization bound:

\begin{proof}[Proof of Theorem \ref{theorem about optimization bound related to W 0 U 0 }]
  For every $1\leq t \leq T$ we have by Lemma \ref{bounding the norm of W 0 and U 0 during SGD} that:
  \begin{align}\label{bound on loss between W_0 and W_t}
    \left|L_D(W_t,U^*) - L_D(W_0,U^*)\right|
    \leq \|U^*\|\cdot \mathbb{E}_x \left[\left\|\sigma(W_tx)-\sigma(W_0x)\right\|\right] \leq \frac{C}{\sqrt{r}}2Lt\epsilon \leq \epsilon
  \end{align}
Where the last inequality is by the choice of $r$. We define the function $g_t(U) := l(N(W_t,U,x_t),y_t)$ where $(x_t,y_t)$ is the example sampled at round $t$ of SGD. Observe that:
    \[  |g_t(U)-g_t(U')| \leq \|\sigma(W_tx)\|\cdot \|U - U'\| \]
where we used the fact that the loss is 1-Lipschitz. Also note that $g_t(U)$ are convex for every $t$. Using Lemma \ref{bounding the norm of W 0 and U 0 during SGD} again:
 \[ |\sigma(W_tx)| \leq |\sigma(W_tx) - \sigma(0)| + |\sigma(0)| \leq L(\sqrt{r} + 1) + L\leq 2L\sqrt{r} \]
 thus $g_t(U)$ is also $2L\sqrt{r}$-Lipschitz for all $t$. We use Theorem \ref{online learning theorem from SSS} on the functions $g_t$ to get:
  \begin{align}
    \sum_{t=1}^{T}g_t(U_t) \leq \sum_{t=1}^{T}g_t(U^*) + \frac{\|U^*-U_0\|^2}{2\eta} + 8\eta rT L^2
  \end{align}
  Dividing by $T$ we get that:
  \begin{align}\label{bound of 1/T times the sum of g_t(U_t)}
    \frac{1}{T}\sum_{t=1}^{T}g_t(U_t) & \leq \frac{1}{T}\sum_{t=1}^{T}g_t(U^*) + \frac{\|U^*-U_0\|^2}{2\eta T} + 8\eta r
  \end{align}
  Using the lower bound of $T$ we get that:
  \begin{equation}\label{Bound on first part of OGD}
    \frac{\|U^*-U_0\|^2}{2\eta T} \leq \frac{\|U^*\|^2}{2\eta T} + \frac{\|U_0\|^2}{2\eta T} \leq 2\epsilon    
  \end{equation}
  Combining \eqref{bound of 1/T times the sum of g_t(U_t)} with \eqref{Bound on first part of OGD}  and plugging in $\eta$ gives us:
  \begin{equation}\label{Bound of 1/T with epsilons}
    \frac{1}{T}\sum_{t=1}^{T}g_t(U_t)  \leq \frac{1}{T}\sum_{t=1}^{T}g_t(U^*) + 3\epsilon
  \end{equation}
  Observe that taking expectation of $g_t$ with respect to the sampled examples in round $t$ of SGD yields: $\mathbb{E}\left[g_t(U)\right] = L_D(W_t,U)$. Thus, taking expectation on \eqref{Bound of 1/T with epsilons} and using \eqref{bound on loss between W_0 and W_t}:
  \begin{equation}
    \frac{1}{T}\sum_{t=1}^{T}\mathbb{E}\left[L_D(W_t, U_t)\right] \leq L_D(W_0,U^*) + 4\epsilon
  \end{equation}
  Thus there is $1\leq t \leq T$ that satisfies:
  \[\mathbb{E}[L_D(W_t, U_t)] \leq L_D(W_0,U^*) + 4\epsilon \]
 Rescaling $\epsilon$ appropriately finishes the proof.
\end{proof}

\section{Approximating polynomials with ReLU networks}\label{app:relu}
Theorem \ref{Main theorem of the paper with analytic activation} can be modified to also include the ReLU activation which is not analytic. This modification requires to add a bias term and also use a non-standard architecture for the network. For terseness we explain here how it can be done without writing the full proof:\\
We begin with the following network architecture:
\begin{equation*}
    N(W,U,b)=\sum_{i=1}^r u_i[\inner{w_i,x}-b_i]_+ - u_i[\inner{-w_i,x}-b_i]_+ + cu_i\cdot\inner{w_i,x} + cu_i,
\end{equation*}
where $c=\frac{1}{e-1}$ is a normalization term which is added for simplicity. This architecture is similar to a standard feed-forward neural network, but includes duplicated ReLU neurons with a negative sign, and  linear and constant factors. The initialization of $w_i$ and $u_i$ is the same as in Theorem \ref{Main theorem of the paper with analytic activation}, and the bias terms $b_i$ are initialized from a uniform distribution on $[0,1]$.\\
Steps 1 and 2 are similar to those used in the original theorem, with adjustments for the added terms, and also in step 2 the function $g(w)$ should depend additionally on the bias term $g(w,b)$. Thus, we can approximate an integral of the form:
\begin{equation}\label{eq: integral form for relu modified theorem}
\int_{w\in \left[\frac{-1}{\sqrt{d}},\frac{1}{\sqrt{d}}\right]^d} \int_0^1 g(w,b)[\inner{w,x} - b]_+ -g(w,b)[\inner{-w,x} - b]_+ + cg(w,b)\inner{w,x} + cg(w,b)\ db\ dw.
\end{equation}
For any $z\in\mathbb{R}$ with $|z|\leq 1$ we have that:
\begin{equation}\label{eq: relu integral trick}
    \int_0^1 [z-b]_+e^b - [-z-b]_+e^{-b}  + cze^b + ce^b\ db = e^z
\end{equation}
Plugging in $g(w,b)$ into \eqref{eq: integral form for relu modified theorem}, and using the integral from \eqref{eq: relu integral trick} with $z=\inner{w_i,x}$ (note that $|\inner{w_i,x}|^2\leq \|w\|^2\|x\|^2\leq 1)$ we can approximate an integral of the form:
\[ \int_{w\in \left[\frac{-1}{\sqrt{d}},\frac{1}{\sqrt{d}}\right]^d} g(w)\exp(\inner{w,x})\ dw. \]
Now we can use step 3 to finish the proof.\\
The requirement for the extra linear and constant terms are also needed in \cite{klusowski2018approximation}. There it is shown that functions that admits certain Fourier transform representations can be approximated using a combination of ReLUs, with an extra linear and constant factors.
\end{document}